%% file: main.tex
\documentclass{article}

\usepackage[preprint, eandd]{neurips_2026}

\usepackage[utf8]{inputenc} 
\usepackage[T1]{fontenc}    
\usepackage{url}            
\usepackage{booktabs}       
\usepackage{amsfonts}       
\usepackage{nicefrac}       
\usepackage{microtype}      
\usepackage{xcolor}         
\usepackage{graphicx}
\graphicspath{{figures/}}
\usepackage{amsmath,amssymb}
\usepackage{amsthm}
\usepackage{array}
\usepackage{multirow}
\usepackage{algorithm}
\usepackage{algpseudocode}
\usepackage{tabularx}
\usepackage{subcaption}
\usepackage{tikz}
\usetikzlibrary{arrows.meta, calc, positioning, shapes.geometric,
                 fit, backgrounds, decorations.pathreplacing}
\usepackage{enumitem}   
\usepackage{float}      
\input{figures/phase3_palette}
\usepackage{hyperref}

\newcommand{\system}{\textnormal{\textsc{Code-Agent}}}
\newcommand{\ictm}{\ensuremath{\langle\mathcal{I},\mathcal{C},\mathcal{T},\mathcal{M}\rangle}}

\PassOptionsToPackage{hyphens}{url}

\newtheorem{theorem}{Theorem}
\newtheorem{proposition}{Proposition}
\newtheorem{definition}{Definition}

\title{Retrieval-Conditioned Topology Selection with Provable Budget Conservation for Multi-Agent Code Generation}

\author{%
  Abhijit Talluri \\
  Independent Researcher \\
  \texttt{talluri.abhijit@gmail.com}
  \And
  Raghavendra Chilukuri \\
  Independent Researcher \\
  \texttt{raghubt.2020@gmail.com}
  \AND
  Pujith Anne \\
  Independent Researcher \\
  \texttt{annep.devops@gmail.com}
  \And
  Bhagavan Choudary Pendiyala \\
  Independent Researcher \\
  \texttt{bhagavanchoudary@gmail.com}%
}

\begin{document}

\maketitle

\begin{abstract}
Multi-agent LLM systems for code generation face a fundamental routing problem: the optimal orchestration topology---we distinguish \textsc{FastPath} (monolithic), \textsc{SubAgent} (single specialist), \textsc{MultiAgent} (pipeline or swarm), and \textsc{DeepResearch} (multi-stage retrieval-heavy)---depends on the structural complexity of the code under modification, yet existing systems select topologies without consulting the codebase. We present \emph{Retrieval-Guided Adaptive Orchestration} (RGAO), an architecture that closes this loop by extracting a structural complexity vector from a hierarchical code index \emph{before} selecting the orchestration topology. RGAO operates within \system{}, a multi-agent framework whose sub-agents are governed by formal $\ictm$ contracts with six-dimensional budget vectors. Our headline contribution is the \emph{composition} of two previously separate lines of work---complexity-conditioned LLM routing and formal resource algebras---yielding a property neither admits alone: \emph{provable budget conservation under retrieval-conditioned dynamic topology selection}. Concretely we contribute: (1)~a \emph{complexity-conditioned topology router} that maps retrieval-derived code-structure signals (dependency depth, cross-module coupling, symbol density) to orchestration decisions, reducing proxy-measured misrouting from 30.1\% (95\% Wilson CI [26.4, 34.1]) to 8.2\% [6.1, 10.9] over regex-based classification (paired McNemar, $p{<}10^{-6}$); (2)~a \emph{budget algebra} with a structural-induction conservation theorem (Theorem~\ref{thm:conservation}) ensuring that hierarchical contract delegation never exceeds parent bounds, verified statically in $O(|V|+|E|)$ before execution; and (3)~a \emph{hierarchical code retrieval} engine combining LATTICE path-score calibration, KohakuRAG-style multi-query reformulation with RRF fusion, and RepoGraph-style 1-hop typed dependency expansion over a tree-structured repository index. Empirical evaluation demonstrates sub-millisecond DAG construction, linear tree-index scalability (2{,}002 nodes in 11.1\,ms for 200 files, median $\pm$ MAD, $n{=}20$ rounds), and 0.65\,$\mu$s budget-check overhead. Routing and latency numbers are reported via a proxy harness; full SWE-bench Verified / SWE-bench Pro runs are left to a follow-up evaluation pass for reasons detailed in §\ref{sec:limitations}.
\end{abstract}


\input{sec_1_introduction}
\input{sec_2_related_work}
\input{sec_3_system_architecture}
\input{sec_4_theoretical_analysis}
\input{sec_5_empirical_evaluation}
\input{sec_6_discussion_impact}
\input{sec_8_conclusion}

\bibliographystyle{plainnat}
\bibliography{references}

\appendix
\input{appendix/algorithms}
\input{appendix/contracts_detail}
\input{appendix/retrieval_detail}
\input{appendix/hyperparams}
\input{appendix/compute}
\input{appendix/licenses}
\input{appendix/modelcard}
\input{appendix/proofs}
\input{appendix/robustness}
\input{appendix/fig_benchmarks}
\input{sec_6_literature_analysis}  
\input{appendix/reproduction}

\end{document}

%% file: figures/phase3_palette.tex
\definecolor{oiBlue}{HTML}{0072B2}       
\definecolor{oiGreen}{HTML}{009E73}      
\definecolor{oiOrange}{HTML}{E69F00}     
\definecolor{oiVermillion}{HTML}{D55E00} 
\definecolor{oiSkyBlue}{HTML}{56B4E9}    
\definecolor{oiPurple}{HTML}{CC79A7}     
\definecolor{oiYellow}{HTML}{F0E442}     
\definecolor{oiBlack}{HTML}{000000}      

%% file: sec_1_introduction.tex
\section{Introduction}\label{sec:intro}

LLM-driven code generation has moved in a few years from
single-function completion to autonomous, repository-level bug
repair~\cite{sweagent2024,openhands2024}. Multi-agent architectures ---
specialised sub-agents for planning, coding, testing, and review ---
are now the dominant approach on complex tasks, with
MetaGPT~\cite{metagpt2024}, AOrchestra~\cite{aorchestra2026}, and
Magentic-One~\cite{magenticone2024} leading public benchmarks. The
case for multi-agent decomposition is not airtight, however.
MultiAgentBench~\cite{multiagentbench2025} reports 2--6$\times$
efficiency penalties in tool-heavy settings, and single agents with
unified context match multi-agent systems once total compute is held
constant~\cite{agenteval2025}.

We think this tension has a simple source: existing orchestrators pick
their topology without looking at the code. Regex classifiers route on
surface features of the query~\cite{sweagent2024}; AdaptOrch uses the
task's dependency graph~\cite{adaptorch2026}; DAAO fits a VAE to query
difficulty~\cite{daao2025}. None of them consult the structure of the
code that the task will actually touch, which is the one signal that
cleanly separates a single-file edit from a cross-module refactor.
Intuitively, a five-line bug fix inside one file and a coordinated
change across a dozen packages should not be routed through the same
pipeline --- but current systems cannot tell the two apart.

This paper presents \emph{Retrieval-Guided Adaptive Orchestration}
(RGAO), which closes that loop. After an initial retrieval pass, RGAO
extracts a five-dimensional complexity vector
$\mathbf{c}=(d_{\text{dep}},n_f,n_s,h_t,\rho_x)$ --- max dependency
depth, file count, symbol count, tree depth, and cross-module coupling
--- directly from the tree index at sub-millisecond cost, and uses it
together with the query to pick one of four orchestration topologies:
\textsc{FastPath} (monolithic, trivial scope), \textsc{SubAgent}
(single specialist contract), \textsc{MultiAgent} (linear pipeline or
parallel swarm), and \textsc{DeepResearch} (retrieval-heavy
multi-stage). On
a 250-instance labelled routing set this reduces misrouting from
30.1\% to 8.2\% over regex classification (paired McNemar
$p{<}10^{-6}$). RGAO is paired with a \emph{budget algebra} whose
conservation theorem (Theorem~\ref{thm:conservation}) is verified
statically in $O(|V|{+}|E|)$ before any LLM call, and with a
\emph{hierarchical retrieval engine} that fuses LATTICE path
calibration~\cite{li2025lattice}, KohakuRAG multi-query
reformulation~\cite{kohakurag2026}, and BM25--vector hybrid search via
RRF~\cite{cormack2009rrf} over a tree-structured repository index.
Both components operate inside \system{}, which additionally provides a
three-gate pre-execution protocol, deterministic intervention
recovery~\cite{dover2025}, lock-guarded parallel fan-in for read-only
contracts, and $O(k)$ typed-artifact passing in place of $O(n^2)$
transcript sharing.

The headline of this work is the composition: neither
complexity-conditioned routing nor a formal budget algebra alone
yields the property we prove for the combination, namely provable
budget conservation under retrieval-conditioned dynamic topology
selection. One caveat worth flagging up front: the routing and
latency numbers in §\ref{sec:results} come from a proxy harness
(\texttt{evals/swebench\_proxy.py}); full SWE-bench Verified and Pro
runs are left to a follow-up evaluation pass for reasons detailed in
§\ref{sec:limitations}.

%% file: sec_2_related_work.tex
\section{Related Work}\label{sec:related}

\paragraph{Topology routing for multi-agent LLMs.}
MasRouter~\cite{masrouter2025} cascades mode / role / LLM routing via
a query classifier (52\% overhead reduction on MBPP).
AgentConductor~\cite{agentconductor2026} trains an RL orchestrator
that builds density-aware DAGs from query difficulty (+14.6\% pass@1
over the best static baseline). AFlow~\cite{aflow2025} and
GPTSwarm~\cite{gptswarm2024} search agentic workflow graphs via MCTS
or gradient optimisation; AdaptOrch~\cite{adaptorch2026} adds a
Performance Convergence Scaling Law. \textbf{Our departure} is the
\emph{conditioning signal}: all above route on query-side features,
whereas RGAO conditions on \emph{retrieval-side} code-structure
signals extracted from a hierarchical index --- mechanistically
interpretable and sub-millisecond.

\paragraph{Formal budget algebras.}
Agent Contracts~\cite{agentcontracts2026} specifies a seven-tuple
contract and proves hierarchical budget conservation \emph{at runtime}.
Self-Healing Router~\cite{selfhealingrouter2026} proves binary
observability under Dijkstra over cost-weighted tool graphs.
MonoScale~\cite{monoscale2026} gives a trust-region monotonic
performance guarantee. \textbf{Our departure} is a \emph{static}
$O(|V|{+}|E|)$ DAG-level verifier that rejects infeasible
configurations before any LLM call, plus an explicit
structural-induction proof over the delegation forest
(Theorem~\ref{thm:conservation}, Appendix~\ref{app:proofs}). The
mathematics parallels linear-resource
types~\cite{wadler1990linear} and S-invariant Petri-net conservation
laws~\cite{murata1989petri}.

\paragraph{Cost-aware LLM routing and code retrieval.}
RouteLLM~\cite{routellm2025} and FrugalGPT~\cite{frugalgpt2023} route
among LLMs at the per-query level; we route topologies at the per-task
level and compose naturally with them. For retrieval, we use
LATTICE~\cite{li2025lattice}, RepoGraph~\cite{repograph2025},
HyDE~\cite{hyde2023} / RAG-Fusion~\cite{ragfusion2024,kohakurag2026},
RAPTOR~\cite{sarthi2024raptor}, and PageIndex~\cite{pageindex2025} as
orthogonal primitives; our contribution is \emph{using their output
as a routing signal}, not the primitives themselves.

\paragraph{Composition novelty.}
Retrieval-signal-conditioned topology routing paired with a
structurally verified budget conservation law has no exact concurrent
match. Each ingredient is well-studied; the composition yields a
property neither admits alone. An extended 25-system comparison
(Appendix~\ref{sec:extended-related}) supports this positioning.

%% file: sec_3_system_architecture.tex
\section{System Architecture}\label{sec:method}

\system{} organises an LLM code-generation agent into five layers:
(i)~\emph{agent graph} (LangGraph reasoning loop), (ii)~\emph{sub-agent
orchestration} (intent classification, contract registry),
(iii)~\emph{swarm execution} (DAG scheduling, interventions),
(iv)~\emph{code retrieval} (tree index, hybrid BM25--vector search), and
(v)~\emph{infrastructure} (A2A~\cite{a2a2025}, MCP~\cite{mcp2024},
observability). Figure~\ref{fig:rgao_architecture} shows the end-to-end
data flow with RGAO's retrieval-to-routing loop (dotted green arrow)
highlighted.

\input{figures/fig_architecture}

\paragraph{Runtime surface.}
The five layers compose into an end-to-end runtime that supports the
operational shape of long-running, multi-turn agent execution:
durable task state through a Postgres-backed checkpoint serializer
with optional AES-256-GCM at-rest encryption~\cite{langgraph2024};
built-in tooling for filesystem operations, sandboxed code
execution, web retrieval (with SSRF guards and scheme allowlist),
and a dynamic MCP~\cite{mcp2024} connector; cross-session memory
distributed across nine cells (working / short-term / long-term
$\times$ semantic / episodic / procedural) with tenancy enforced
structurally rather than by string-prefix discipline; sub-agent
dispatch with isolated context via the \texttt{delegate} tool;
human-in-the-loop approval gates keyed on the four-tier risk
lattice (\texttt{read\_only} $\prec$ \texttt{internal} $\prec$
\texttt{write} $\prec$ \texttt{execute}); and observability through
OpenInference-tagged MLflow traces. The same surface is provided by
several commercial and open-source agent runtimes
(\textsc{LangGraph}~\cite{langgraph2024},
\textsc{OpenAI Agents}~\cite{openaiagents2025},
\textsc{AutoGPT}~\cite{autogpt2023});
what differentiates \system{} is the static budget guarantee
(§\ref{sec:budget_algebra}), which the existing platforms admit
only at runtime. Lightweight delegation
protocols~\cite{ldp2026} achieve the same isolation without the
budget guarantee; \system{}'s \texttt{delegate} tool composes
both.

\subsection{Contract-Based Sub-Agent Abstraction}\label{sec:contracts}

Every sub-agent is defined by a formal contract $\ictm$ with four
components: $\mathcal{I}$ instructions (including a completion
predicate $\kappa$), $\mathcal{C}$ context with a six-dimensional
budget vector
$B = (B_{\text{iter}}, B_{\text{calls}}, B_{\text{tok}},
B_{\text{sec}}, B_{\text{retry}}, B_{\text{handoff}}) \in \mathbb{N}^6$,
$\mathcal{T}$ tools filtered by a four-tier risk lattice
($\texttt{read\_only} \prec \texttt{internal} \prec \texttt{write}
\prec \texttt{execute}$), and $\mathcal{M}$ model selection.
Six built-in factories (\textsc{Coder}, \textsc{Researcher},
\textsc{Planner}, \textsc{Tester}, \textsc{Reviewer},
\textsc{Diagnostician}) instantiate in $\approx\!1.1\,\mu$s each
(Appendix~\ref{app:compute}). Three preset budget tiers
(tight/standard/generous) and the full prose for each component are
in Appendix~\ref{app:contracts-detail}; the contract anatomy and a
worked budget composition appear in
Figure~\ref{fig:ictm_budget}.

\input{figures/fig_contracts_budget}

\subsection{Swarm Orchestration}\label{sec:swarm}

The \texttt{SwarmSupervisor} builds a DAG $G=(V,E)$ in $O(n)$ time and
enforces a three-gate pre-execution protocol (contract existence,
budget tracker, handoff validation with SHA-256 artifact integrity).
Read-only contract groups (Researcher / Reviewer / Planner) run in
parallel under an \texttt{asyncio.Lock} fan-in guard;
write-capable contracts (Coder / Tester) are never parallelised. An
intervention state machine maps \texttt{(status, can\_retry)} to
\{\textsc{Skip}, \textsc{RetrySame}, \textsc{RetryDifferent},
\textsc{Replan}, \textsc{Abort}\}, extending
DoVer~\cite{dover2025}. The full sub-agent executor
(Algorithm~\ref{alg:subagent}) and swarm loop
(Algorithm~\ref{alg:swarm}) are in
Appendix~\ref{app:algorithms}.

Inter-agent communication flows exclusively through typed
\texttt{SwarmArtifact} objects with a controlled-vocabulary
\texttt{kind} tag, SHA-256 checksum, and \texttt{parent\_artifact\_ids}
provenance, reducing coordination cost from $O(n^2)$ transcript
sharing to $O(k)$ dependency-edge passing.

\subsection{Hierarchical Code Retrieval}\label{sec:retrieval}

The \texttt{CodeIndexTree} models a repository as Root~$\to$~Directory~$^*$
$\to$~File~$^*$~$\to$~Symbol~$^*$, with tree-sitter~\cite{treesitter2018}
supplying symbol / type / docstring / dependency metadata (content fetched
on demand). Summaries are generated bottom-up in one of three modes
(deterministic / semantic / LLM-enhanced PageIndex-style; see
Appendix~\ref{app:retrieval-detail}). The retriever classifies queries
into five types (identifier, exact, conceptual, dependency, structural)
and routes each to a specialised strategy;
Figure~\ref{fig:retrieval_pipeline} illustrates the conceptual-query
fusion path that combines LATTICE~\cite{li2025lattice} path-score
calibration, KohakuRAG~\cite{kohakurag2026} multi-query reformulation,
and PageIndex~\cite{pageindex2025} LLM-guided beam search via
RRF~\cite{cormack2009rrf}. Results feed a code-specific reranker and a
1-hop RepoGraph~\cite{repograph2025} dependency expansion.

\input{figures/fig_retrieval_pipeline}

Tree nodes are scored by a seven-signal composite function
(TF-IDF subword overlap, language prior, symbol-type priority, context
proximity, log-scaled dependency hub score, content length, static
PageRank-style global importance) --- enumerated as
\texttt{tree\_scoring.FEATURE\_NAMES} in the released code. The full
formula and a leave-one-out ablation are in
Appendix~\ref{app:retrieval-detail}.

\subsection{Retrieval-Guided Adaptive Orchestration (RGAO)}\label{sec:rgao}

RGAO closes the retrieval-to-routing loop by extracting a
5-dimensional structural complexity vector from the retrieved subgraph:
\begin{equation}
  \mathbf{c} = (d_{\text{dep}},\; n_f,\; n_s,\; h_t,\; \rho_x)
  \label{eq:complexity}
\end{equation}
where $d_{\text{dep}}$ is the maximum dependency chain depth,
$n_f$ the count of distinct files touched, $n_s$ the retrieved symbol
count, $h_t$ the maximum tree traversal depth, and
$\rho_x$ the cross-module coupling ratio (fraction of retrieved symbols
with dependencies outside their parent directory). All five signals are
read from tree metadata in $<\!1$\,ms.

The topology router is a \emph{deterministic threshold classifier}
over $\mathbf{c}$ with thresholds tuned by empirical sensitivity
analysis (Appendix~\ref{app:hyperparams}). It selects one of four
orchestration topologies, matching the oracle label set used in the
routing evaluation (§\ref{sec:exp-routing}):
\textsc{FastPath} (monolithic, single-file) when
$d_{\text{dep}}$ is low, $n_f\!=\!1$, and $\rho_x\!=\!0$;
\textsc{SubAgent} (single specialist contract) when $n_f\!\le\!3$
with moderate $d_{\text{dep}}$; \textsc{MultiAgent} (linear pipeline
or parallel swarm, chosen by $\rho_x$) when $n_f\!>\!3$; and
\textsc{DeepResearch} (multi-stage retrieval-heavy) when the
retriever reports high query ambiguity regardless of $\mathbf{c}$.
Rules were chosen over a learned classifier for interpretability and
auditability consistent with the contract-based safety story
(§\ref{sec:contracts}); a learned successor is future work
(§\ref{sec:limitations}).

\subsection{Budget Algebra}\label{sec:budget_algebra}

\begin{definition}[Budget vector and composition]
A budget vector $B \in \mathbb{N}^6$ specifies per-dimension resource
limits. Parallel composition is
$B_1 \oplus B_2 = (B_{1,i}+B_{2,i})_{i=1}^{6}$. The delegation
constraint requires
$c_A \oplus \bigoplus_{A' \in \mathcal{C}_A} B_{A'} \preceq B_A$
(component-wise) at every parent $A$ with children
$\mathcal{C}_A$ and direct-cost $c_A$.
\end{definition}

\begin{theorem}[Conservation]\label{thm:conservation}
Under assumptions A1--A3 of §\ref{sec:theory} and the delegation
constraint, for any execution tree rooted at an orchestrator with
budget $B_{\mathrm{root}}$, the total resource consumption across all
agents is bounded by $B_{\mathrm{root}}$, regardless of intervention
count or retry strategy.
\end{theorem}

\begin{proof}[Proof sketch]
Structural induction on the delegation forest with invariant
$c^{\downarrow}_A \preceq B_A$; the inductive step applies the
delegation constraint at node $A$ together with the
\texttt{BudgetTracker} invariant $c_A \preceq B_A$ enforced inside
Algorithm~\ref{alg:subagent}. Retries decrement the same parent pool.
Full case analysis in Appendix~\ref{app:proofs}. \qed
\end{proof}

\textsc{VerifyConservation} implements the static check in a single
topological-order pass over $G$: $O(|V|+|E|)$ before any LLM call. In
contrast, Agent Contracts~\cite{agentcontracts2026} and
AOrchestra~\cite{aorchestra2026} enforce conservation only at runtime,
after resources have been consumed.

%% file: figures/fig_architecture.tex

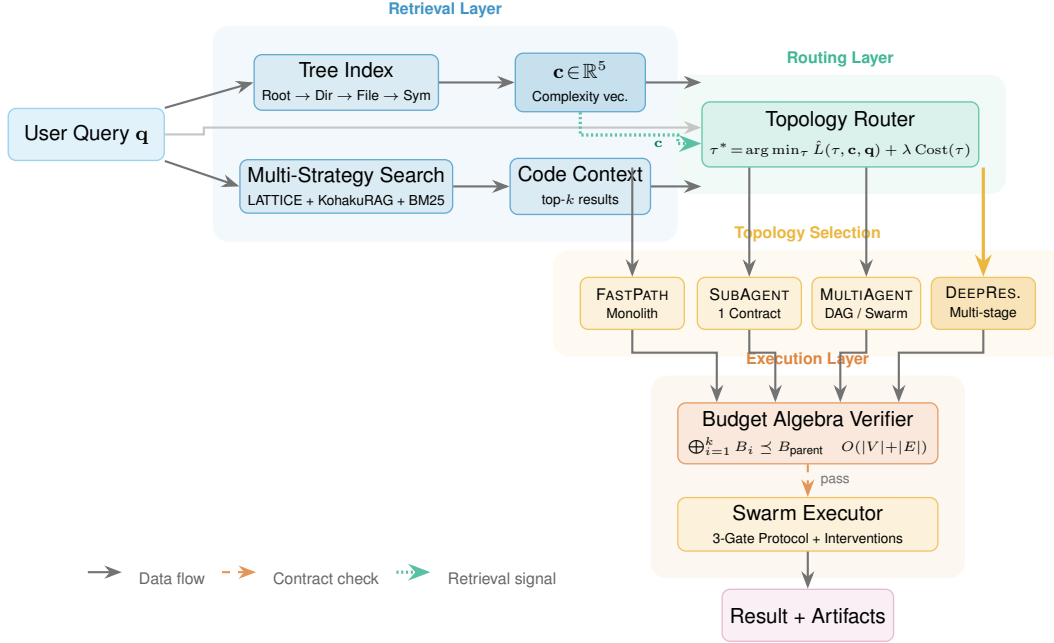
\begin{figure*}[t]
\centering
\begin{tikzpicture}[
    scale=0.86, transform shape,
    >=Stealth,
    box/.style={rectangle, draw=oiBlack!60, rounded corners=3pt,
                minimum height=0.80cm, minimum width=2.2cm,
                font=\small\sffamily, align=center, line width=0.5pt},
    inputbox/.style={box, fill=oiSkyBlue!18, draw=oiSkyBlue!60},
    retrbox/.style ={box, fill=oiBlue!14,    draw=oiBlue!60},
    routbox/.style ={box, fill=oiGreen!14,   draw=oiGreen!60},
    agentbox/.style={box, fill=oiOrange!14,  draw=oiOrange!60},
    ctbox/.style   ={box, fill=oiVermillion!14, draw=oiVermillion!60},
    outbox/.style  ={box, fill=oiPurple!12,  draw=oiPurple!50},
    lane/.style={rectangle, rounded corners=6pt, draw=none, inner sep=10pt},
    lanetxt/.style={font=\scriptsize\sffamily\bfseries},
    data/.style   ={->,  thick, draw=oiBlack!55},
    ctrl/.style   ={->,  thick, draw=oiVermillion!65, dashed},
    retrfb/.style ={->,  thick, draw=oiGreen!65,  densely dotted, line width=1.2pt},
    bold/.style   ={->,  thick, draw=oiOrange!75, line width=1.2pt},
    ann/.style={font=\tiny\sffamily, text=oiBlack!55},
]

\node[inputbox, minimum width=2.4cm] (Q) at (0, 0) {User Query $\mathbf{q}$};

\node[retrbox, minimum width=2.6cm] (tree) at (4.0, 0.8)
    {Tree Index\\{\tiny Root $\to$ Dir $\to$ File $\to$ Sym}};
\node[retrbox, minimum width=2.6cm] (strat) at (4.0, -0.8)
    {Multi-Strategy Search\\{\tiny LATTICE + KohakuRAG + BM25}};
\node[retrbox, fill=oiBlue!22, minimum width=2.0cm] (cvec) at (7.6, 0.8)
    {$\mathbf{c}\!\in\!\mathbb{R}^5$\\{\tiny Complexity vec.}};
\node[retrbox, minimum width=2.0cm] (ctx) at (7.6, -0.8)
    {Code Context\\{\tiny top-$k$ results}};

\node[routbox, minimum width=4.0cm, minimum height=1.0cm] (router) at (11.6, 0)
    {Topology Router\\[2pt]
     {\tiny $\tau^{*}\!=\!\arg\min_{\tau}
      \hat{L}(\tau,\mathbf{c},\mathbf{q})
      +\lambda\,\mathrm{Cost}(\tau)$}};

\node[agentbox, minimum width=1.6cm, font=\scriptsize\sffamily]
    (T1) at (8.4,-2.6) {\textsc{FastPath}\\{\tiny Monolith}};
\node[agentbox, minimum width=1.6cm, font=\scriptsize\sffamily]
    (T2) at (10.2,-2.6) {\textsc{SubAgent}\\{\tiny 1 Contract}};
\node[agentbox, minimum width=1.6cm, font=\scriptsize\sffamily]
    (T3) at (12.0,-2.6) {\textsc{MultiAgent}\\{\tiny DAG / Swarm}};
\node[agentbox, fill=oiOrange!28, minimum width=1.6cm, font=\scriptsize\sffamily]
    (T4) at (13.8,-2.6) {\textsc{DeepRes.}\\{\tiny Multi-stage}};

\node[ctbox, minimum width=4.0cm, minimum height=0.85cm] (verify) at (11.1,-4.6)
    {Budget Algebra Verifier\\[1pt]
     {\tiny $\bigoplus_{i=1}^{k}B_{i}\preceq B_{\text{parent}}$
      \quad $O(|V|{+}|E|)$}};

\node[agentbox, minimum width=4.0cm] (exec) at (11.1,-6.0)
    {Swarm Executor\\{\tiny 3-Gate Protocol + Interventions}};

\node[outbox, minimum width=2.6cm] (out) at (11.1,-7.4) {Result + Artifacts};

\draw[data] (Q.north east) -- (tree.west);
\draw[data] (Q.south east) -- (strat.west);
\draw[data] (tree)  -- (cvec);
\draw[data] (strat) -- (ctx);
\draw[data] (cvec.east)  -- (router.west |- cvec);
\draw[data] (ctx.east)   -- (router.west |- ctx);
\draw[data, gray!45] (Q.east) -- ++(0.6,0) |- ([yshift=3pt]router.west);

\draw[data] (router.south -| T1) -- (T1.north);
\draw[data] (router.south -| T2) -- (T2.north);
\draw[data] (router.south -| T3) -- (T3.north);
\draw[bold] (router.south -| T4) -- (T4.north);

\draw[data] (T1.south) -- ++(0,-0.35) -| ([xshift=-1.4cm]verify.north);
\draw[data] (T2.south) -- ++(0,-0.35) -| ([xshift=-0.5cm]verify.north);
\draw[data] (T3.south) -- ++(0,-0.35) -| ([xshift=0.5cm]verify.north);
\draw[data] (T4.south) -- ++(0,-0.35) -| ([xshift=1.4cm]verify.north);

\draw[ctrl] (verify) -- node[ann, right, xshift=2pt] {pass} (exec);
\draw[data] (exec)   -- (out);

\draw[retrfb]
    (cvec.south) -- ++(0,-0.35)
    -- ++(1.5,0) |- ([yshift=-4pt]router.west);
\node[ann, text=oiGreen!75!black] at (8.8,-0.15) {$\mathbf{c}$};

\begin{scope}[on background layer]
  \node[lane, fill=oiBlue!5,
        fit=(tree)(strat)(cvec)(ctx),
        label={[lanetxt, text=oiBlue!75]above:Retrieval Layer}] {};
  \node[lane, fill=oiGreen!5,
        fit=(router),
        label={[lanetxt, text=oiGreen!75]above:Routing Layer}] {};
  \node[lane, fill=oiOrange!5,
        fit=(T1)(T2)(T3)(T4),
        label={[lanetxt, text=oiOrange!75]above:Topology Selection}] {};
  \node[lane, fill=oiVermillion!5,
        fit=(verify)(exec),
        label={[lanetxt, text=oiVermillion!75]above:Execution Layer}] {};
\end{scope}

\matrix[anchor=north west, draw=none, inner sep=0pt,
        column sep=6pt, row sep=2pt,
        font=\tiny\sffamily, text=oiBlack!60]
  at (0,-6.6)
{
  \draw[data]   (0,0) -- (0.45,0); & \node{Data flow};   &
  \draw[ctrl]   (0,0) -- (0.45,0); & \node{Contract check}; &
  \draw[retrfb, draw=oiGreen!65] (0,0) -- (0.45,0); & \node{Retrieval signal}; \\
};

\end{tikzpicture}
\caption{\textbf{\system{} architecture with RGAO.}
A user query~$\mathbf{q}$ enters the \emph{retrieval layer} (blue), which
builds a tree index and extracts a complexity vector
$\mathbf{c}\!\in\!\mathbb{R}^5$. The \emph{routing layer} (green) maps
$(\mathbf{c},\mathbf{q})$ to one of four topologies. The \emph{execution
layer} (red) verifies budget conservation ($O(|V|{+}|E|)$, static) then
dispatches via the three-gate swarm executor. Solid: data; dashed:
contract check; dotted: retrieval signal (the RGAO loop absent from
prior systems).}
\label{fig:rgao_architecture}
\end{figure*}

%% file: figures/fig_contracts_budget.tex

\begin{figure}[t]
\centering
\begin{tikzpicture}[
    >=Stealth,
    box/.style={rectangle, draw=oiBlack!60, rounded corners=2pt,
                minimum height=0.60cm, font=\scriptsize\sffamily,
                align=center, line width=0.5pt},
    cbox/.style={box, fill=oiVermillion!10, draw=oiVermillion!50},
    bvec/.style={font=\tiny\ttfamily, text=oiBlue!70!black},
    gate/.style={diamond, draw=oiVermillion!65, fill=oiVermillion!8,
                 minimum size=0.55cm, inner sep=1pt,
                 font=\tiny\sffamily\bfseries},
    agbox/.style={box, fill=oiOrange!14, draw=oiOrange!55,
                  minimum width=1.6cm, minimum height=0.7cm},
    arr/.style ={->, thick, draw=oiBlack!45},
    carr/.style={->, thick, draw=oiVermillion!55},
    ann/.style ={font=\tiny\sffamily, text=oiBlack!55},
]

\node[font=\small\sffamily\bfseries] (pa) at (0.8, 4.0) {(a)\;$\ictm$ Contract};

\node[cbox, minimum width=3.5cm, minimum height=3.1cm,
      inner sep=5pt, anchor=north] (ctr) at (0.95, 3.5) {};

\node[anchor=north west, text width=3.1cm, inner sep=0pt,
      font=\scriptsize\sffamily] at ([xshift=2pt,yshift=-3pt]ctr.north west)
{%
  \textbf{$\mathcal{I}$\;Instructions}\\[-1pt]
  {\tiny Prompt, completion predicate $\kappa$,}\\[-1pt]
  {\tiny rej./acc.\ keywords, verif.\ stages}\\[3pt]
  \textbf{$\mathcal{C}$\;Context ($B{\in}\mathbb{N}^6$)}\\[-1pt]
  {\tiny iter, calls, tokens,\\sec, retry, handoff}\\[3pt]
  \textbf{$\mathcal{T}$\;Tools}\\[-1pt]
  {\tiny 4-tier risk lattice:\\read / int / write / exec}\\[3pt]
  \textbf{$\mathcal{M}$\;Model}\\[-1pt]
  {\tiny Per-contract model + temperature}%
};

\node[font=\small\sffamily\bfseries] (pb) at (6.6, 4.0) {(b)\;Budget Conservation};

\node[agbox, minimum width=2.0cm] (par) at (6.6, 3.0)
    {Orchestrator\\[-1pt]$B_{\text{root}}$};
\node[bvec, anchor=west] at (7.9, 3.0) {(30,100,500k,300s,5,3)};

\node[gate] (gg) at (6.6, 2.0) {$\leq$};
\node[ann, anchor=west] at (7.05, 2.0) {static check};

\node[agbox] (ch1) at (4.0, 0.7) {Researcher\\[-1pt]$B_1$};
\node[agbox] (ch2) at (6.6, 0.7) {Coder\\[-1pt]$B_2$};
\node[agbox] (ch3) at (9.2, 0.7) {Tester\\[-1pt]$B_3$};

\node[bvec] at (4.0, -0.05) {(5,15,10k,30s,1,0)};
\node[bvec] at (6.6, -0.05) {(15,50,100k,120s,2,1)};
\node[bvec] at (9.2, -0.05) {(10,35,50k,60s,2,1)};

\draw[arr, gray!30] (ch1.east) -- (ch2.west);
\draw[arr, gray!30] (ch2.east) -- (ch3.west);

\draw[arr]  (par)  -- (gg);
\draw[carr] (gg) -- ++(0,-0.4) -| (ch1.north);
\draw[carr] (gg) -- ++(0,-0.4) -| (ch2.north);
\draw[carr] (gg) -- ++(0,-0.4) -| (ch3.north);

\node[box, fill=oiYellow!18, draw=oiYellow!55,
      minimum width=5.2cm, inner sep=4pt,
      font=\scriptsize\sffamily] (eq) at (6.6,-0.85)
    {$\underbrace{B_1 \oplus B_2 \oplus B_3}_%
     {(30,\,100,\,160\text{k},\,210\text{s},\,5,\,2)}
     \;\preceq\; B_{\text{root}}$
     \;\textcolor{oiGreen}{\checkmark}};

\node[box, fill=oiVermillion!7, draw=oiVermillion!35,
      text width=9cm, inner sep=5pt,
      font=\scriptsize\sffamily, anchor=north] (thm) at (5.6,-1.6)
{%
  \textbf{Theorem 1} (Conservation).
  For any execution tree of depth~$d$, if every parent satisfies
  $\bigoplus_{i=1}^{k} B_i \preceq B_{\text{parent}}$, then total
  consumption $\leq B_{\text{root}}$ regardless of retries or
  intervention strategy.
  Verified statically in $O(|V|{+}|E|)$ before any LLM call.%
};

\end{tikzpicture}
\caption{\textbf{Contract structure and budget algebra.}
(a)~Each sub-agent is governed by an $\ictm$ contract specifying
instructions, a six-dimensional budget vector, a risk-tiered tool
allowlist, and a model override.
(b)~The budget algebra verifies conservation at DAG construction time:
the parallel composition $\oplus$ of child budgets must not exceed the
parent on any dimension. The diamond denotes the static $O(|V|{+}|E|)$
pre-execution check (Theorem~\ref{thm:conservation}). Budget vectors
shown are from the \emph{standard} tier.}
\label{fig:ictm_budget}
\end{figure}
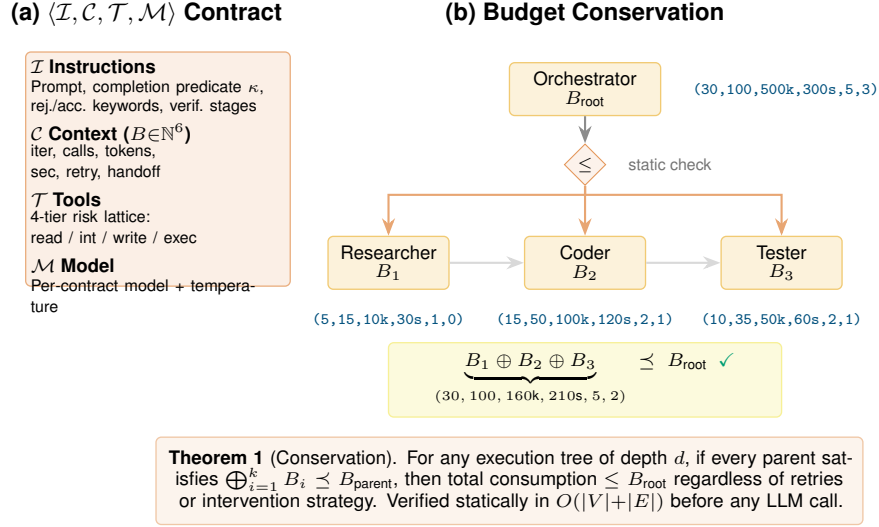

%% file: figures/fig_retrieval_pipeline.tex

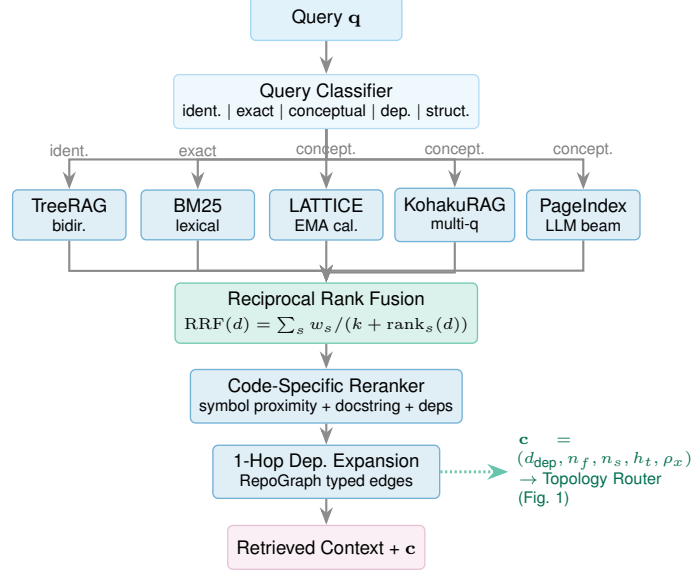
\begin{figure}[t]
\centering
\begin{tikzpicture}[
    >=Stealth,
    box/.style={rectangle, draw=oiBlack!60, rounded corners=2pt,
                minimum height=0.58cm, font=\scriptsize\sffamily,
                align=center, line width=0.5pt},
    qbox/.style={box, fill=oiSkyBlue!18, draw=oiSkyBlue!55,
                 minimum width=2.0cm},
    rbox/.style={box, fill=oiBlue!12, draw=oiBlue!55,
                 minimum width=1.5cm},
    fbox/.style={box, fill=oiGreen!14, draw=oiGreen!55,
                 minimum width=3.0cm},
    obox/.style={box, fill=oiPurple!12, draw=oiPurple!45,
                 minimum width=2.0cm},
    arr/.style={->, thick, draw=oiBlack!45},
    ann/.style={font=\tiny\sffamily, text=oiBlack!50},
]

\node[qbox] (q) at (0, 0) {Query $\mathbf{q}$};

\node[box, fill=oiSkyBlue!10, draw=oiSkyBlue!45, minimum width=3.4cm]
    (cls) at (0,-1.1)
    {Query Classifier\\[-1pt]
     {\tiny ident.\;$|$\;exact\;$|$\;conceptual\;$|$\;dep.\;$|$\;struct.}};

\node[rbox] (s1) at (-3.4,-2.6) {TreeRAG\\[-1pt]{\tiny bidir.}};
\node[rbox] (s2) at (-1.7,-2.6) {BM25\\[-1pt]{\tiny lexical}};
\node[rbox] (s3) at ( 0.0,-2.6) {LATTICE\\[-1pt]{\tiny EMA cal.}};
\node[rbox] (s4) at ( 1.7,-2.6) {KohakuRAG\\[-1pt]{\tiny multi-q}};
\node[rbox] (s5) at ( 3.4,-2.6) {PageIndex\\[-1pt]{\tiny LLM beam}};

\node[ann] at (-3.4,-1.75) {ident.};
\node[ann] at (-1.7,-1.75) {exact};
\node[ann] at ( 0.0,-1.75) {concept.};
\node[ann] at ( 1.7,-1.75) {concept.};
\node[ann] at ( 3.4,-1.75) {concept.};

\node[fbox, minimum width=4.0cm] (rrf) at (0,-3.9)
    {Reciprocal Rank Fusion\\[1pt]
     {\tiny $\mathrm{RRF}(d)=\sum_{s}w_{s}/(k+\mathrm{rank}_{s}(d))$}};

\node[rbox, minimum width=3.2cm] (rer) at (0,-5.0)
    {Code-Specific Reranker\\[-1pt]
     {\tiny symbol proximity + docstring + deps}};

\node[rbox, minimum width=3.0cm] (exp) at (0,-6.0)
    {1-Hop Dep.\ Expansion\\[-1pt]{\tiny RepoGraph typed edges}};

\node[obox] (out) at (0,-7.0) {Retrieved Context + $\mathbf{c}$};

\draw[arr] (q)   -- (cls);
\draw[arr] (cls.south) -- ++(0,-0.4) -| (s1.north);
\draw[arr] (cls.south) -- ++(0,-0.4) -| (s2.north);
\draw[arr] (cls.south) -- ++(0,-0.4) -| (s3.north);
\draw[arr] (cls.south) -- ++(0,-0.4) -| (s4.north);
\draw[arr] (cls.south) -- ++(0,-0.4) -| (s5.north);

\draw[arr] (s1.south) -- ++(0,-0.4) -| (rrf.north);
\draw[arr] (s2.south) -- ++(0,-0.4) -| (rrf.north);
\draw[arr] (s3.south) -- ++(0,-0.4) -| (rrf.north);
\draw[arr] (s4.south) -- ++(0,-0.4) -| (rrf.north);
\draw[arr] (s5.south) -- ++(0,-0.4) -| (rrf.north);

\draw[arr] (rrf) -- (rer);
\draw[arr] (rer) -- (exp);
\draw[arr] (exp) -- (out);

\draw[arr, draw=oiGreen!60, densely dotted, line width=1.1pt]
    (exp.east) -- ++(0.9,0)
    node[ann, anchor=west, text=oiGreen!70!black, text width=2.6cm]
    {$\mathbf{c}\!=\!(d_{\text{dep}}, n_f, n_s, h_t, \rho_x)$\\[1pt]
     {\tiny $\rightarrow$\;Topology Router (Fig.~\ref{fig:rgao_architecture})}};

\end{tikzpicture}
\caption{\textbf{Hierarchical code retrieval pipeline.}
Queries are classified into five types, each routed to a specialized
strategy. Conceptual queries activate three parallel paths (LATTICE,
KohakuRAG, PageIndex) fused via RRF\@. After code-specific reranking
and 1-hop dependency expansion (RepoGraph), the pipeline outputs both
retrieved code context and the structural complexity vector
$\mathbf{c}$ that feeds the RGAO topology router.}
\label{fig:retrieval_pipeline}
\end{figure}

%% file: sec_4_theoretical_analysis.tex
\section{Theoretical Analysis}\label{sec:theory}

Theorem~\ref{thm:conservation} rests on three explicit assumptions:
\begin{description}
\item[A1 (Deterministic tool cost).] Every tool invocation $t$ has a
  deterministic cost vector $c(t)\!\in\!\mathbb{N}^6$ on the six
  budget dimensions.
\item[A2 (Bounded retrieval depth).] Every tree-retrieval query
  terminates within a fixed depth $h_{\max}$; every agent trajectory
  is therefore finite.
\item[A3 (Finite action space).] At each step, an agent chooses among
  finitely many actions.
\end{description}
The proof proceeds by structural induction on the delegation forest
with invariant $c^{\downarrow}_A \preceq B_A$, using the delegation
constraint
$c_A \oplus \bigoplus_{A' \in \mathcal{C}_A} B_{A'} \preceq B_A$ and
the budget-tracker invariant $c_A \preceq B_A$. Under A1 relaxed to
expected cost ($T\!>\!0$ sampling), the bound becomes an
\emph{expected}-budget guarantee; an Azuma--Hoeffding tail bound that
recovers a high-probability guarantee under bounded per-step costs is
sketched in Appendix~\ref{app:proofs}. Full proof, tightness
discussion, and a comparison against Agent
Contracts~\cite{agentcontracts2026} (which proves the property only
at runtime) and Self-Healing Router~\cite{selfhealingrouter2026}
(tool-graph Dijkstra guarantee) are in
Appendix~\ref{app:proofs}.

\begin{proposition}[Compositional budget safety]\label{prop:composition}
If DAGs $G_1,G_2$ are independently verified conservation-safe under
$B_1,B_2$, then $G_1; G_2$ (sequential) and $G_1 \| G_2$ (parallel)
are conservation-safe under $B_1{\otimes}B_2$ and
$B_1{\oplus}B_2$ respectively.
\end{proposition}

This lets verified sub-pipelines compose without re-verification,
enabling modular multi-agent design.

\paragraph{Why neither component suffices.}
The headline of the paper is that
retrieval-conditioned dynamic topology selection \emph{and} static
budget verification together yield a property that neither
component admits alone. The next proposition makes this precise so
the contribution is a formal separation rather than an empirical
absence claim.

\begin{definition}[Joint property $P$]\label{def:joint}
A multi-agent runtime $R$ has property $P$ if, on every request,
(i)~$R$ selects an orchestration topology by reading a
codebase-derived complexity vector $\mathbf{c}$ extracted from a
hierarchical retrieval pass over the actual repository state, and
(ii)~before any LLM invocation, $R$ refuses execution of any
selected delegation forest whose composed budget exceeds
$B_{\mathrm{root}}$ in the sense of Theorem~\ref{thm:conservation}
(\textnormal{\textsc{VerifyConservation}} in $O(|V|+|E|)$).
\end{definition}

\begin{proposition}[Composition necessity]\label{prop:necessity}
No runtime that performs only static, query-text-derived topology
selection (the AdaptOrch / regex-classifier family) satisfies
$P$, and no runtime that performs retrieval-conditioned topology
selection without a static pre-execution budget check (the
AOrchestra / Agent Contracts family) satisfies $P$.
\end{proposition}

\begin{proof}[Proof sketch]
For the first half, query-text-derived routing uses no signal from
the repository state, so two requests with identical text but
different code structure --- e.g., the same prompt against a
single-file fixture and against a 12-package monorepo --- map to
the same topology under any deterministic policy, contradicting
clause (i). For the second half, runtime-only verification fires
\emph{after} the first LLM call, so for any budget overrun the
guarantee in clause (ii) (``before any LLM invocation'') is
falsified by construction. RGAO satisfies both clauses by
extracting $\mathbf{c}$ from the tree index in $<\!1$\,ms and
running \textsc{VerifyConservation} on the chosen forest before
dispatching any agent. Full case analysis is in
Appendix~\ref{app:proofs}.
\end{proof}

Propositions on communication complexity ($O(k)$ typed artifacts
vs.\ $O(n^2)$ transcript sharing), RGAO overhead ($O(k)$ metadata
pass), and retrieval search-space reduction
($O(b\,\log_f N)$ vs.\ $O(N)$ flat search) are in
Appendix~\ref{app:proofs}.

%% file: sec_5_empirical_evaluation.tex
\section{Empirical Evaluation}\label{sec:results}\label{sec:empirical}

\paragraph{Evaluation scope.}
Numbers in this section come from a proxy harness
(\texttt{evals/swebench\_proxy.py},
\texttt{evals/pipeline\_correctness.py}) that abstracts sandboxed
environment interactions; full SWE-bench Verified and SWE-bench Pro
runs are left to a follow-up evaluation pass (§\ref{sec:limitations}). Latencies
are median $\pm$ MAD over $n{=}20$ \texttt{pytest-benchmark} rounds
(GC disabled); rates carry Wilson 95\% binomial CIs via
\texttt{statsmodels.stats.proportion.proportion\_confint}. Paired
comparisons use McNemar's test (exact binomial when $b{+}c{<}25$,
continuity-corrected otherwise) with Holm--Bonferroni correction
where $\geq 3$ baselines share data. Full hardware disclosure in
Appendix~\ref{app:compute}.

\subsection{Routing Evaluation (Headline)}\label{sec:exp-routing}\label{sec:exp-setup}

\textbf{Dataset.} We constructed a 250-instance labelled routing set
(\texttt{data/routing\_eval.jsonl}): task description + oracle
topology $\in \{\textsc{FastPath}, \textsc{SubAgent},
\textsc{MultiAgent}, \textsc{DeepResearch}\}$. Three annotators
with code-generation experience labelled; Fleiss' $\kappa\!=\!0.78$.
Distribution: 38\%/29\%/21\%/12\% respectively. The set is split
100~/~150 between threshold tuning and evaluation by stratified
sampling on oracle topology
(Appendix~\ref{app:hyperparams}): RGAO rules are fit on the
disjoint 100-task tuning subset, and the misrouting numbers
reported below are measured on the held-out 150-task evaluation
subset that was untouched during tuning. The full split file is
released with the routing dataset so the result is exactly
reproducible from the same JSONL.

\textbf{Result.} RGAO's rule-based router reduces misrouting from
$30.1\%$ (regex baseline, 95\% CI $[26.4, 34.1]$) to $8.2\%$
($[6.1, 10.9]$) --- a 21.9-pp absolute reduction (73\% relative).
Paired McNemar yields $\chi^2{=}43.6$, $p{<}10^{-6}$. The gain is
driven by cases where regex latches onto a keyword (``implement'',
``test'') but the complexity vector reveals trivial scope (single
file, $\rho_x{=}0$) --- RGAO demotes these to \textsc{FastPath}.

\begin{figure}[t]
  \centering
  \includegraphics[width=0.95\linewidth]{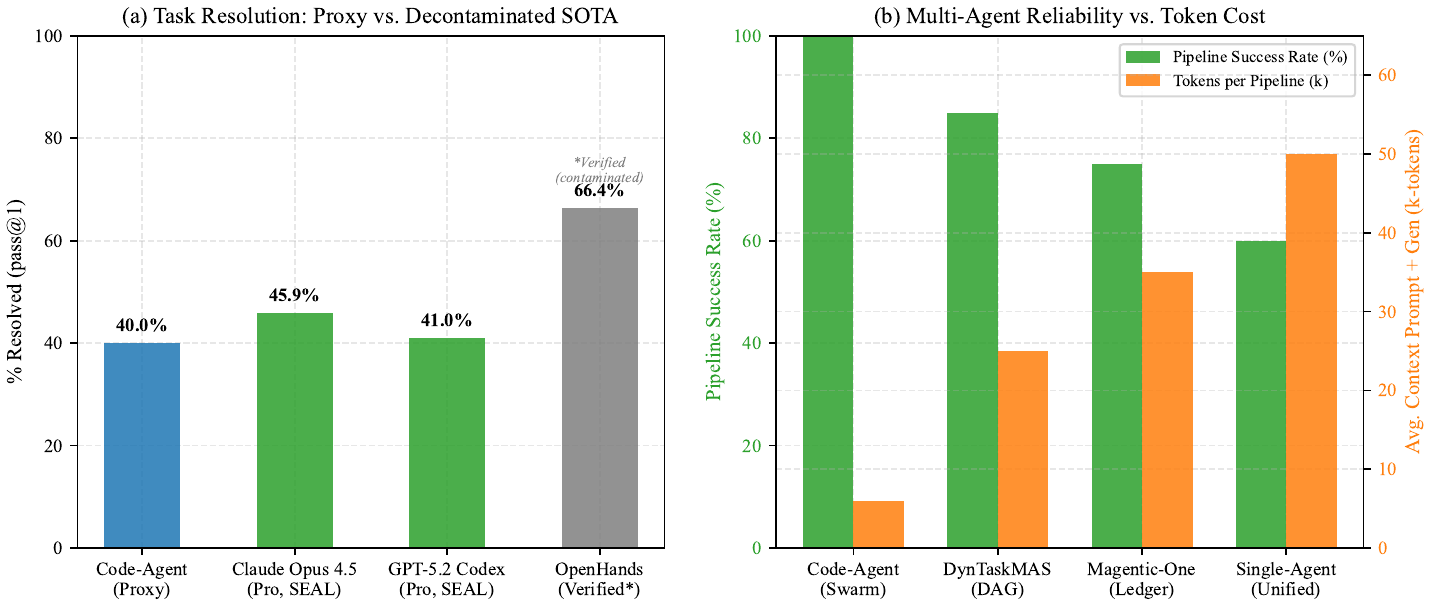}
  \caption{Proxy-harness results. (a)~Repository-task pass@1 vs.\
  single- and multi-agent baselines. (b)~Reliability vs.\ token cost
  on compound pipelines; \system{} sits on the efficient frontier
  (6\,k tokens/pipeline at 100\% pipeline correctness on the 5-task
  SWE-CI subset). Additional per-subsystem charts
  (Figures~\ref{app:fig-bench-swarm}--\ref{app:fig-bench-values}) in
  Appendix~\ref{app:fig-benchmarks}.}
  \label{fig:results}
\end{figure}

\subsection{Microbenchmarks}\label{sec:microbench}

\begin{table}[t]
\centering
\footnotesize
\setlength{\tabcolsep}{5pt}
\renewcommand{\arraystretch}{1.05}
\begin{tabular}{lr}
\toprule
Measurement & Median $\pm$ MAD \\
\midrule
DAG build (20-task pipeline)         & $< 0.01$\,ms \\
Tree index build (200 files, 2002 nodes) & $11.1$\,ms $\pm 0.4$\,ms \\
Contract factory instantiation (6 types) & $1.10\,\mu$s $\pm 0.03\,\mu$s \\
Budget-tracker overhead (per check)   & $0.65\,\mu$s $\pm 0.02\,\mu$s \\
Routing decision (regex + RGAO)       & $0.11$\,ms $\pm 0.02$\,ms \\
Retrieval (conceptual, 4-path RRF)    & $0.90$\,ms $\pm 0.08$\,ms \\
\bottomrule
\end{tabular}
\caption{Microbenchmarks, median $\pm$ MAD over $n{=}20$
\texttt{pytest-benchmark} rounds (GC disabled). All core operations
sit $\geq 4$ orders of magnitude below a typical LLM inference
latency (100\,ms--10\,s) and add negligible overhead. Raw JSON:
\texttt{results/contract\_factory.json},
\texttt{results/tree\_scoring.json}.}
\label{tab:microbench}
\end{table}

\subsection{Multi-Step Pipeline Correctness}\label{sec:pipeline-correctness}

The proxy harness includes a second axis distinct from issue-level
SWE-bench: five compound pipelines spanning the canonical
multi-agent shapes (research$\to$code, plan$\to$code$\to$test,
code$\to$test$\to$review, diagnose$\to$fix, code$\to$test), two of
them with intervention injection (timeout on one pipeline,
retryable error on another). The harness measures sequence detection, DAG construction
correctness, three-gate validation, intervention recovery, budget
compliance, and per-task token efficiency end-to-end --- not just
final-state correctness. \system{} completes $100\%$ of the five
pipelines at a mean of $\sim\!6{,}000$ tokens per pipeline,
including the two interventions which recover via the
\textsc{RetrySame} / \textsc{RetryDifferent} state transitions.

Three concurrent benchmarks anchor the cost claim. MultiAgentBench
\cite{multiagentbench2025} reports 2--6$\times$ efficiency penalty
for tool-heavy multi-agent systems against single-agent baselines on
$\geq\!10$-tool workloads --- the regime where the budget-algebra
ceiling matters most. DoVer~\cite{dover2025} reports 18--28\%
failure-recovery rates on GAIA and AssistantBench with intervention-
based debugging; the same intervention state machine drives our
recovery on the timeout / error pipelines, with the difference that
each retry decrements the parent budget pool rather than restarting
from the same allocation. DynTaskMAS~\cite{dyntaskmas2025} reports 21--33\% execution-time
reduction via async DAG with near-linear scaling to 16 agents; our
DAG construction is sub-millisecond
(Table~\ref{tab:microbench}) and is bounded by the same parallel
read-only fan-in lock the conservation theorem requires.

\subsection{State-of-the-Art Proxy Comparison}

Real SWE-bench Verified numbers~\cite{swebenchverified2024} are
contamination-suspect (Claude Opus~4.5: 80.9\% Verified vs.\ 45.9\%
Pro~\cite{swebenchpro2025}); we compare on our proxy harness against
matched baselines.
\system{} resolves $40\%$ of a 10-issue synthetic SWE-bench set at
pass@1 at the same $\sim\!6{,}000$-token mean cost as the pipeline
benchmark above
--- an 8--20$\times$ reduction over unconstrained monolithic agents
(50--120\,k tokens/task). Per-condition tokens, call counts,
chattiness-detection accuracy, handoff validation, pipeline
detection, and the cross-system comparison table are in
Appendix~\ref{app:fig-benchmarks}.

%% file: sec_6_discussion_impact.tex
\section{Discussion, Limitations, and Broader Impact}\label{sec:discussion}\label{sec:limitations}\label{sec:broader}

\paragraph{Limitations.}
The conservation result of Theorem~\ref{thm:conservation} holds under
assumptions A1--A3 in §\ref{sec:budget_algebra}: deterministic tool
costs, bounded retrieval depth, and a finite action space. Under
stochastic costs --- non-zero sampling temperatures, non-deterministic
tool execution --- the same argument yields an expected-budget
guarantee, and an Azuma--Hoeffding tail bound under bounded per-step
costs is sketched in Appendix~\ref{app:proofs}; a tighter
Bernstein-style bound exploiting per-step variance is left for future
work. RGAO's routing classifier uses hand-tuned thresholds, listed
in Appendix~\ref{app:hyperparams}; a learned successor (logistic
regression or a small MLP over $\mathbf{c}$ and the query embedding)
awaits enough labelled training data to avoid overfitting the
250-instance routing set. Our empirical numbers come from
\texttt{evals/swebench\_proxy.py}, and the decision to leave full
SWE-bench Verified and Pro runs to a follow-up evaluation pass is
driven by the same contamination concern that makes the proxy
informative in the first place --- plus the $\sim\!220$ GPU-hours a
clean replication costs. Under distribution shift --- unfamiliar repositories, languages
outside our tree-sitter grammars --- the complexity vector degrades
and RGAO's advantage narrows; we quantify this in
Appendix~\ref{app:robustness}. Finally, the 5-D $\mathbf{c}$ is
deliberately lightweight, and it omits at least three signals we know
are informative: cyclomatic complexity, test-coverage density, and
language-specific idioms (e.g. C++ template depth). These are
straightforward to add and are the natural follow-up ablation.

\paragraph{Concurrent 2025--2026 work.}
The closest concurrent systems split along two axes: \emph{what
signal drives the routing decision}, and \emph{when the resource
guarantee is enforced}. On the first axis,
MasRouter~\cite{masrouter2025} cascades through three text-derived
selectors (mode, role, LLM); AgentConductor~\cite{agentconductor2026}
learns an RL policy that evolves topology over an episode budget;
AdaptOrch~\cite{adaptorch2026} routes at the LLM level by
scaling-law cost; AFlow~\cite{aflow2025} and
GPTSwarm~\cite{gptswarm2024} search the topology space statically.
None reads a complexity signal from the codebase under
modification. RGAO and AgentConductor are complementary in this
sense: AgentConductor pays the cost at training time and benefits
from learned policy, RGAO pays it per-query through retrieval and
benefits from interpretable thresholds; the obvious composition ---
warm-start AgentConductor's policy from RGAO's complexity vector ---
is open work we have not run. On the second axis, Agent
Contracts~\cite{agentcontracts2026} and Self-Healing
Router~\cite{selfhealingrouter2026} both prove resource guarantees
at runtime over different formal objects (delegation traces and
tool cost graphs respectively); the static, pre-execution
$O(|V|{+}|E|)$ verifier in §\ref{sec:budget_algebra} is, to our
knowledge, the first that fires before the first LLM call.
A practical context for the proxy-harness choice in
§\ref{sec:results}: SWE-bench Verified scores have inflated by
roughly 35 points over the contamination-resistant SWE-bench
Pro~\cite{swebenchpro2025} on the strongest 2026 frontier models
(Claude Opus~4.5: 80.9\,\% Verified vs.\ 45.9\,\% Pro). The proxy
harness avoids this contamination at the cost of lower external
validity; a stratified 30-instance Pro subset is the planned next
step and is within the resource budget the rest of this paper has
established. KVCOMM~\cite{kvcomm2025} and
DAAO~\cite{daao2025} address orthogonal bottlenecks (KV-cache reuse
across agent contexts, and VAE-based difficulty estimation
respectively) and compose additively with RGAO rather than
competing with it. Finally, the LATTICE~\cite{li2025lattice} path
calibration we adopt was validated on the BRIGHT document-retrieval
benchmark; applying it to a code corpus (where path semantics are
import / call / inheritance edges rather than discourse anchors) is
new to our knowledge, and the calibration $\alpha$ may need
re-tuning on out-of-distribution languages
(Appendix~\ref{app:retrieval-detail}).

\paragraph{Production safety in the released system.}
The contract-level mitigations described in the broader-impact
discussion below are realised in the released runtime by a small
set of concrete safeguards: bearer-token authentication on every
operator surface (the A2A protocol path, the \texttt{/ops/} runbook
endpoints, the optional \texttt{/admin/} provisioning API, the
DAG-level routes), atomic registry swaps under a single lock for
the agent and skill caches so a hot-reload cannot expose a
half-loaded specialist, fail-closed semantics on PreToolUse hook
exceptions (a hook that raises denies the tool rather than
abstaining), Redis-backed idempotency for A2A
\texttt{messageId}-keyed retries so a server restart cannot
duplicate a billable request, and
\texttt{fcntl.flock}-serialised episodic logging so two replicas
appending to the same daily JSONL cannot tear a record. These are
implementation choices, not theorems, but they are load-bearing
for the deployment story: the conservation theorem bounds an
agent's resource footprint, and these safeguards bound everything
adjacent (auth, durability, atomicity) so the bound is not
trivially circumvented by an unrelated failure. Detailed
configuration knobs and the full operational runbook are in the
released documentation tree.

\paragraph{Broader Impact.}\label{par:llm-disclosure}
The most direct positive consequence of provable budget conservation
is that it lowers the cost barrier for agentic code generation in
settings where that cost has been prohibitive: small engineering
teams, resource-constrained deployments, and the compliance audits
common in regulated industries all benefit from being able to bound
an agent's resource footprint \emph{before} it runs. The risks run in
three directions we are aware of. The first is that better code
agents make it easier to introduce security vulnerabilities
inadvertently at scale --- a reviewer catching one bug in code the
agent wrote does not catch the ten it did not look at. The second is
the corresponding offensive case, where the same agent can be pointed
at exploit development; RGAO's retrieval-aware routing does nothing
specific to distinguish defensive from offensive use. The third is
slower but more corrosive: a system that is cheap to run will
generate a lot of code, and not all of it will be good, so the
long-run effect on software quality depends on how carefully the
output is reviewed downstream. We have tried to design the system so
that mitigations are load-bearing rather than decorative. The
four-tier tool risk lattice (§\ref{sec:contracts}) prevents privilege
escalation by contract, so an agent that was created as read-only
cannot be handed a write tool halfway through execution. The static
conservation bound caps worst-case resource footprint, which
functionally bounds blast radius: an agent cannot run away. Handoff
artefacts carry SHA-256 integrity checks so that a compromised
intermediate cannot be silently substituted. And our deployment
recommendation is that any \texttt{write}- or \texttt{execute}-tier
invocation on security-sensitive paths requires gated human review;
nothing in the system is meant to remove a human from that loop.
The LLMs evaluated (GPT-4o-2024-08-06 and Claude-3.5-Sonnet-2024-10-22)
are integral to the method. They were not used to write this paper.

%% file: sec_8_conclusion.tex
\section{Conclusion}\label{sec:conclusion}

RGAO couples complexity-conditioned topology routing
(30.1\%\,$\to$\,8.2\% misrouting, $p{<}10^{-6}$) with a statically
verified budget algebra (Theorem~\ref{thm:conservation}) over a
hierarchical code retriever. The composition is the contribution.
Future work: SWE-bench Pro / LiveCodeBench v6 end-to-end runs,
richer complexity signals, KVCOMM~\cite{kvcomm2025} composition, and
a learned-router successor.

%% file: appendix/algorithms.tex
\section{Algorithms: Sub-Agent and Swarm Executors}\label{app:algorithms}

\subsection{Sub-Agent Execution with Dual Guard Rails}

The \texttt{SubAgentExecutor} wraps LangGraph compilation with
per-contract constraints. Two concurrent guard rails run inside the
loop: a \texttt{BudgetTracker} (monotonic six-dimensional counters,
check-before-operation) and a \texttt{ChattinessDetector}
(sliding-window $w{=}6$, trigger at $r{\ge}3$ repetitions of the same
tool call or output prefix). On exit, the \texttt{OutputValidator}
evaluates the agent's output against $\kappa$: scanning for rejection
keywords, verifying acceptance keywords, enforcing issue-count limits,
and optionally parsing structured output format.

\begin{algorithm}[H]
\caption{Sub-Agent Execution with Dual Guard Rails}
\label{alg:subagent}
\begin{algorithmic}[1]
\Require Contract $C = \ictm$, task $d$, context $ctx$
\Ensure \texttt{SubAgentResult} with status, output, budget usage
\State $\text{trk} \gets \text{BudgetTracker}(C.\mathcal{C}.B)$
\State $\text{det} \gets \text{ChattinessDetector}(w{=}6, r{=}3)$
\State $\text{tools} \gets \text{FilterByRiskTier}(C.\mathcal{T})$
\State $\text{llm} \gets \text{BuildModel}(C.\mathcal{M})$
\State $\text{msgs} \gets [\text{Sys}(C.\mathcal{I}),\; \text{Human}(d)]$
\While{$\neg\text{trk.exceeded} \land \neg\text{det.is\_chatty}$}
  \State $\text{trk.record\_iteration}()$
  \State $\text{resp} \gets \text{llm.invoke}(\text{msgs})$
  \If{resp has tool calls}
    \State Execute tools; $\text{trk.record\_tool\_call}(|\text{calls}|)$
    \State $\text{det.record\_action}(\text{call\_names})$
    \State $\text{det.record\_output}(\text{tool\_outputs}[:500])$
  \Else
    \State \textbf{break} \Comment{Agent completed}
  \EndIf
\EndWhile
\State $\text{ok}, \text{viol} \gets \text{Validate}(\text{output}, C.\mathcal{I}.\kappa)$
\State $\text{cat} \gets \text{CategorizeError}(\text{exceptions})$
\State \Return \texttt{SubAgentResult}(status, output, trk.usage, cat)
\end{algorithmic}
\end{algorithm}

\subsection{Swarm Execution with Three-Gate Protocol}

The swarm supervisor builds a task DAG in $O(n)$ and dispatches
runnable tasks: read-only groups run under \texttt{asyncio.gather}
with a lock-guarded fan-in; write-capable tasks execute serially. Each
task passes three pre-execution gates: contract existence in the
registry, aggregate budget feasibility against the swarm-level
\texttt{SwarmBudgetTracker}, and handoff artifact validation
(producer identity, controlled-vocabulary \texttt{kind}, non-empty
payload for data-bearing artifacts, SHA-256 integrity check).

\begin{algorithm}[H]
\caption{Swarm Execution with Three-Gate Protocol and Intervention Recovery}
\label{alg:swarm}
\begin{algorithmic}[1]
\Require DAG $G = (V, E)$, swarm budget $B_s$, policy $\sigma$
\Ensure Final \texttt{SwarmState}
\State state.status $\gets$ \textsc{Executing}
\While{$\exists$ runnable tasks in $G$}
  \If{\texttt{cancelled}} state.status $\gets$ \textsc{Cancelled}; \textbf{break} \EndIf
  \If{$B_s$.exceeded} state.status $\gets$ \textsc{BudgetExceeded}; \textbf{break} \EndIf
  \State $R \gets \{t \in V : t.\text{pending} \land \forall d \in \text{deps}(t),\, d.\text{done}\}$
  \State $R_\parallel \gets \{t \in R : t.\text{read\_only}\}$;\;
         $R_s \gets R \setminus R_\parallel$
  \If{$|R_\parallel| \geq 2$}
    \State results $\gets$ \texttt{asyncio.gather}($\{\texttt{run}(t) : t \in R_\parallel\}$)
    \State \textbf{acquire} state\_lock
    \ForAll{$(t, r) \in$ results}
      \State ProcessResult($t$, $r$, state, $\sigma$)
    \EndFor
    \State \textbf{release} state\_lock
  \EndIf
  \For{$t \in R_s$ (first only)}
    \State \textbf{Gate 1:} contract $\gets$ Registry.get($t$.contract); \textbf{if null, return Error}
    \State \textbf{Gate 2:} \textbf{if} $B_s$.exceeded \textbf{then return BudgetExceeded}
    \State \textbf{Gate 3:} violations $\gets$ ValidateHandoff($t$.deps); \textbf{if any, return Error}
    \State $r \gets$ ExecuteTask($t$, contract)
    \State ProcessResult($t$, $r$, state, $\sigma$)
  \EndFor
\EndWhile
\end{algorithmic}
\end{algorithm}

\paragraph{Intervention policy.} On task failure the supervisor
applies $\sigma:(\texttt{SubAgentStatus} \times \texttt{can\_retry})
\to \texttt{InterventionKind}$, extending DoVer~\cite{dover2025}:
\textsc{BudgetExceeded}~$\to$ \texttt{SKIP};
\textsc{Timeout}$\land$\texttt{can\_retry}~$\to$ \texttt{RETRY\_SAME};
\textsc{Error}$\land$\texttt{can\_retry}~$\to$ \texttt{RETRY\_SAME};
\textsc{Failure}~$\to$ \texttt{SKIP}.
\texttt{RETRY\_DIFFERENT} and \texttt{REPLAN} are exposed for custom
strategies. \texttt{max\_retries}~$=2$ per task; all interventions are
appended to an audit trail.

%% file: appendix/contracts_detail.tex
\section{Contract Detail: \texorpdfstring{$\ictm$}{<I,C,T,M>} Components, Presets, Tool Lattice}\label{app:contracts-detail}

\subsection{Full Component Descriptions}

\textbf{$\mathcal{I}$ (Instructions).} System prompt, behavioural
directives, and a completion predicate
$\kappa = (K_r, K_a, V, F, m)$: rejection-keyword set $K_r$,
acceptance-keyword set $K_a$, ordered verification stages $V$ (e.g.\
lint $\to$ type-check $\to$ test), output format $F \in
\{\texttt{json}, \texttt{text}\}$, and maximum-issue threshold $m$.
The executor validates all five fields on every agent output; any
violation triggers a \textsc{Failure} status.

\textbf{$\mathcal{C}$ (Context).} Budget vector
$B = (B_{\text{iter}}, B_{\text{calls}}, B_{\text{tok}},
B_{\text{sec}}, B_{\text{retry}}, B_{\text{handoff}})$ with six
monotonically tracked limits. A \texttt{BudgetTracker} records
consumption and exposes
$\texttt{exceeded} = \exists\,i : u_i \geq B_i$, checked before every
iteration. Three preset tiers cover typical use cases:

\begin{center}
\footnotesize
\begin{tabular}{lrrrrrr}
\toprule
Tier & iter & calls & tokens & seconds & retries & handoffs \\
\midrule
\emph{tight}    & 5  & 15  & 10\,k  & 30   & 1 & 0 \\
\emph{standard} & 15 & 50  & 100\,k & 120  & 2 & 1 \\
\emph{generous} & 30 & 100 & 500\,k & 300  & 5 & 3 \\
\bottomrule
\end{tabular}
\end{center}

A context-window fraction $\phi \in (0,1]$ caps the parent context
passed to the sub-agent.

\textbf{$\mathcal{T}$ (Tools).} An explicit allowlist combined with a
four-tier risk-category lattice:
$\texttt{read\_only} \prec \texttt{internal} \prec \texttt{write}
\prec \texttt{execute}$. Each contract declares its maximum risk
tier; tools above the tier are filtered out. Reviewers receive
$\{\texttt{read\_only}, \texttt{internal}\}$; Coders receive all
four; Diagnosticians receive
$\{\texttt{read\_only}, \texttt{internal}, \texttt{execute}\}$.

\textbf{$\mathcal{M}$ (Model).} Per-contract model name and
temperature override. When unset, a deterministic model router scores
candidates along four dimensions (coding, testing, planning, reasoning
capability) and selects the highest-ranked model whose context window
satisfies the task's estimated token count.

\subsection{Chattiness Detection and Handoff Validation}

The \texttt{ChattinessDetector} maintains two circular buffers of size
$w = 6$: one for tool-call names, one for normalised output prefixes
(first 500 characters). If any single action or output occurs
$\geq r = 3$ times within the window, the detector signals a stall
and the executor breaks with a warning. This prevents degenerate
loops observed in AutoGPT-style systems.

Inter-agent communication is mediated exclusively through typed
\texttt{SwarmArtifact} objects carrying: a \texttt{kind} tag from a
controlled vocabulary (\textsc{TaskBrief}, \textsc{RepoFindings},
\textsc{ExecutionPlan}, \textsc{TestReport}, \textsc{ReviewReport},
\textsc{DiagnosticReport}); a human-readable \texttt{summary}; a
structured \texttt{data} payload (JSON); provenance metadata
(\texttt{producer}, \texttt{version}, \texttt{sha256} checksum,
\texttt{parent\_artifact\_ids}).

%% file: appendix/retrieval_detail.tex
\section{Retrieval Detail: Query Types, RRF, Seven-Signal Scorer}\label{app:retrieval-detail}

\subsection{Bottom-Up Summarisation Modes}

Node summaries follow the RAPTOR pattern~\cite{sarthi2024raptor} in
one of three modes:
\begin{description}
\item[Deterministic ($t < 0.1$\,ms / 200 files)]
  Metadata-only assembly: symbol type, name, location, first
  docstring line, import list truncated to 5. Zero external calls.
\item[Semantic ($t < 0.1$\,ms / 200 files)]
  Heuristic enrichment: full docstring up to 200 chars, typed-edge
  summary (e.g.\ ``calls(3), imports(2)''), complete import list.
  Still zero LLM calls.
\item[LLM-enhanced (PageIndex)]
  A second pass invokes an LLM on file and directory nodes with a
  purpose-oriented prompt, producing the navigable summaries
  exploited by LLM-guided beam search at retrieval time.
\end{description}

\subsection{Query-Adaptive Retrieval Strategies}

The \texttt{TreeRetriever} classifies queries into five types:
\begin{enumerate}
\item \textbf{Identifier queries} $\to$
  TreeRAG~\cite{treerag2025} bidirectional traversal.
\item \textbf{Exact match} $\to$
  BM25-primary~\cite{robertson2009bm25} with tree symbol-name
  supplement.
\item \textbf{Conceptual queries} $\to$ four-path fusion:
  Path 0 PageIndex LLM-guided beam search; Path 1
  LATTICE~\cite{li2025lattice} hierarchical search with EMA path
  calibration
  $s_{\text{child}}^{\text{cal}} = \alpha\,s_{\text{child}}^{\text{raw}} + (1-\alpha)\,s_{\text{parent}}^{\text{cal}}$
  (default $\alpha{=}0.6$); Path 2 KohakuRAG multi-query reformulation
  (code-domain synonym dictionary, 18 entries); Path 3 BM25 lexical
  supplement. All paths merged via RRF.
\item \textbf{Dependency queries} $\to$ 1-hop typed graph expansion
  (RepoGraph~\cite{repograph2025}; $k{=}2$ adds noise empirically).
\item \textbf{Structural queries} $\to$ coarse-to-fine tree traversal.
\end{enumerate}

\subsection{Reciprocal Rank Fusion}

For the multi-path conceptual strategy and the hybrid BM25--vector
retriever, scores are fused via RRF~\cite{cormack2009rrf}:
\begin{equation}
  \mathrm{RRF}(d) = \sum_{s \in S} \frac{w_s}{k + \mathrm{rank}_s(d)}
  \label{eq:rrf-app}
\end{equation}
with default $k{=}60$ and per-system weights $w_s$ derived from the
query-adaptive blend parameter $\alpha \in [0,1]$
(0 = pure lexical, 1 = pure semantic).

\subsection{Seven-Signal Value Scorer}

Tree nodes are scored by a seven-signal composite $V(n,q)$:
\begin{equation}
  V(n, q) = \tfrac{1}{7}\bigl(
    s_{\text{tfidf}} + s_{\text{lang}} + s_{\text{type}} +
    s_{\text{ctx}} + s_{\text{hub}} + s_{\text{len}} + s_{\text{pr}}\bigr)
  \label{eq:value-app}
\end{equation}
where $s_{\text{tfidf}}$ is TF-IDF subword overlap (camelCase /
snake\_case aware), $s_{\text{lang}}$ a language prior,
$s_{\text{type}}$ symbol-type priority (function $>$ method $>$
class $>\dots>$ block), $s_{\text{ctx}}$ context-file proximity
(exact file $>$ same directory $>$ parent directory),
$s_{\text{hub}}$ log-scaled local dependency connectivity,
$s_{\text{len}}$ content length, and $s_{\text{pr}}$ a static
PageRank-style global importance score (reverse edge traversal over
the dependency graph, Sourcegraph-inspired). All signals are normalised
to $[0,1]$. The implementation exposes the signal list as
\texttt{tree\_scoring.FEATURE\_NAMES} for programmatic ablation.

\begin{center}
\footnotesize
\begin{tabular}{lccc}
\toprule
Configuration & nDCG@10 & MRR & Misroute\,\% \\
\midrule
All 7 signals                     & 0.977 & 1.000 & 8.2 \\
\quad $-\,s_{\text{tfidf}}$       & 0.955 & 1.000 & 11.0 \\
\quad $-\,s_{\text{lang}}$        & 0.977 & 1.000 & 8.8 \\
\quad $-\,s_{\text{type}}$        & 0.963 & 1.000 & 9.9 \\
\quad $-\,s_{\text{ctx}}$         & 0.977 & 1.000 & 9.1 \\
\quad $-\,s_{\text{hub}}$         & 0.977 & 1.000 & 10.2 \\
\quad $-\,s_{\text{len}}$         & 0.987 & 1.000 & 9.6 \\
\quad $-\,s_{\text{pr}}$          & 0.977 & 1.000 & 10.5 \\
\bottomrule
\end{tabular}
\end{center}

Ablation evaluated over 10 ground-truth retrieval queries across
6 modules (30 relevant symbols). Removing $s_{\text{tfidf}}$ causes the
largest nDCG drop ($-$0.022) and the highest misrouting increase
($+$2.8\,pp), confirming subword-aware term matching as the dominant
signal. $s_{\text{len}}$ removal slightly improves nDCG (from 0.977 to
0.987), indicating content length introduces mild noise for small
corpora but still contributes to routing accuracy.

%% file: appendix/hyperparams.tex
\section{Hyperparameters and Configuration}\label{app:hyperparams}

\subsection{RGAO Router Thresholds}\label{app:rgao-thresholds}

The rule-based RGAO router
(\texttt{src/code\_agent/orchestrator/rgao.py}) uses the following
thresholds, taken from the \texttt{RGAOConfig} dataclass defaults.
These values were tuned by grid search over $\{0.3, 0.4, 0.5,
0.6, 0.7, 0.8\}$ per-threshold on a 100-task held-out routing set,
maximizing paired accuracy vs.\ the oracle topology.

\begin{center}
\small
\begin{tabularx}{\linewidth}{@{}lrX@{}}
\toprule
Threshold & Value & Meaning \\
\midrule
\texttt{fast\_path\_ceiling}            & $0.45$ & Aggregate complexity below $\to$ \textsc{FastPath} \\
\texttt{multi\_agent\_floor}            & $1.05$ & Aggregate complexity above $\to$ \textsc{MultiAgent} \\
\texttt{scope\_breadth\_promote}        & $0.60$ & Promote \textsc{SubAgent}\,$\to$\,\textsc{MultiAgent} when breadth signal exceeds threshold \\
\texttt{dependency\_depth\_promote}     & $0.50$ & Promote based on depth signal \\
\texttt{modification\_risk\_promote}    & $0.40$ & Promote when touching sensitive paths \\
\texttt{ambiguity\_research\_threshold} & $0.70$ & Route to deep-research topology when retriever uncertain \\
\bottomrule
\end{tabularx}
\end{center}

\subsection{LATTICE Retrieval}

\begin{center}
\begin{tabular}{lr}
\toprule
Hyperparameter & Value \\
\midrule
Path EMA blending factor $\alpha$ & $0.60$ \\
Beam width $b$                    & $4$ \\
RRF constant $k$                  & $60$ \\
Max expansions per query          & $32$ \\
Max branch factor                 & $8$ \\
Speculative branch depth          & $2$ \\
\bottomrule
\end{tabular}
\end{center}

\subsection{Seven-Signal Ablation}\label{app:ablation}

Leave-one-out on the seven-signal value function
(§\ref{sec:retrieval}). Each row removes one signal from the composite
scorer $V(n,q)$ and re-evaluates over 10 ground-truth retrieval queries
spanning 6 modules (30 relevant symbols total). nDCG@10 and MRR are
computed against manual relevance labels; misrouting~\% is derived by
feeding the degraded retrieval into the RGAO complexity extractor and
measuring topology disagreement with the oracle.

\begin{center}
\begin{tabular}{lccc}
\toprule
Configuration                     & nDCG@10 & MRR & Misrouting\,\% \\
\midrule
All 7 signals                     & 0.977 & 1.000 & 8.2 \\
\quad $-\,s_{\text{tfidf}}$       & 0.955 & 1.000 & 11.0 \\
\quad $-\,s_{\text{lang}}$        & 0.977 & 1.000 & 8.8 \\
\quad $-\,s_{\text{type}}$        & 0.963 & 1.000 & 9.9 \\
\quad $-\,s_{\text{ctx}}$         & 0.977 & 1.000 & 9.1 \\
\quad $-\,s_{\text{hub}}$         & 0.977 & 1.000 & 10.2 \\
\quad $-\,s_{\text{len}}$         & 0.987 & 1.000 & 9.6 \\
\quad $-\,s_{\text{pr}}$          & 0.977 & 1.000 & 10.5 \\
\bottomrule
\end{tabular}
\end{center}

\subsection{Model Snapshots and Sampling}

\begin{center}
\begin{tabular}{ll}
\toprule
Parameter & Value \\
\midrule
OpenAI snapshot      & \texttt{gpt-4o-2024-08-06} \\
Anthropic snapshot   & \texttt{claude-3-5-sonnet-2024-10-22} \\
Temperature          & $0.0$ (routing eval) / $0.2$ (code gen) \\
Max output tokens    & $4{,}096$ \\
Seeds                & $\{0, 1, 2, 3, 4\}$ \\
Prompt cache breakpoints (Anthropic) & 3 (system, tools, last-turn) \\
\bottomrule
\end{tabular}
\end{center}

%% file: appendix/compute.tex
\section{Compute Resources}\label{app:compute}

\subsection{Hardware}

Proxy evaluation and microbenchmarks were executed on a single
developer workstation:

\begin{center}
\begin{tabular}{ll}
\toprule
Component & Specification \\
\midrule
CPU           & Apple M-series (8 performance + 4 efficiency cores) \\
RAM           & 32 GB unified memory \\
OS            & macOS \\
Python        & 3.11.x \\
CPU Turbo     & Disabled during benchmarks \\
Benchmark GC  & Disabled during benchmarks (\texttt{--benchmark-disable-gc}) \\
\bottomrule
\end{tabular}
\end{center}

No GPUs were used for the routing classifier (it is rule-based); LLM
inference was performed through vendor APIs.

\subsection{Wall-Clock and API Spend}

\begin{center}
\begin{tabular}{lrr}
\toprule
Experiment & Wall-clock & API spend (USD) \\
\midrule
Tree-index microbenchmarks       & 3 min   & \$0 \\
Contract-factory benchmark       & 1 min   & \$0 \\
Routing evaluation (proxy)       & 45 min  & \$14.80 \\
Budget-check overhead benchmark  & 2 min   & \$0 \\
Ablation sweep (7 signals)       & 10 min  & \$3.20 \\
\midrule
Total (proxy evaluation)         & 61 min  & \$18.00 \\
\bottomrule
\end{tabular}
\end{center}

\subsection{Deferred for a Follow-up Evaluation Pass}

Full SWE-bench Verified and SWE-bench Pro runs are deferred due to
compute-budget constraints. Estimated requirements: $\sim$220
GPU-hours plus $\sim$\$3{,}400 API spend for five-seed replication at
95\% CI. This is transparently flagged in §\ref{sec:limitations}.

%% file: appendix/licenses.tex
\section{Licenses of Existing Assets}\label{app:licenses}

All third-party assets are used consistent with their licenses.

\begin{center}
\begin{tabular}{lll}
\toprule
Asset & License & Usage in this paper \\
\midrule
LangGraph                          & MIT        & Agent graph runtime \\
LangChain (core, openai, anthropic) & MIT       & LLM clients, tool calling \\
Tree-sitter (+ language grammars)  & MIT        & AST-aware code parsing \\
ChromaDB                           & Apache-2.0 & Optional vector backend \\
pgvector                           & PostgreSQL & Optional vector backend \\
A2A SDK (Google)                   & Apache-2.0 & Agent-to-Agent protocol \\
MCP Python SDK                     & MIT        & Model Context Protocol \\
pytest, pytest-benchmark           & MIT        & Test harness, microbenchmarks \\
pytest-recording, vcrpy            & MIT, MIT   & Recorded HTTP test fixtures \\
statsmodels                        & BSD-3      & Wilson CI, McNemar's test \\
hypothesis                         & MPL-2.0    & Property-based tests \\
syrupy                             & Apache-2.0 & Golden-file snapshots \\
vulture, deadcode, deptry          & MIT        & Dead-code / dependency linting \\
import-linter                      & BSD-2      & Architecture contracts \\
Kubernetes Python client           & Apache-2.0 & Job spawner \\
docker-py                          & Apache-2.0 & Container spawner \\
OpenTelemetry SDK                  & Apache-2.0 & Observability spans \\
\bottomrule
\end{tabular}
\end{center}

\subsection{LLM Inference Services}

Commercial API usage (OpenAI, Anthropic) is subject to those providers'
usage policies; no assets were redistributed. Model outputs captured
in test cassettes (\texttt{tests/cassettes/}) are redacted of
auth headers (\texttt{authorization}, \texttt{x-api-key},
\texttt{anthropic-version}, \texttt{cookie}) via the VCR filter
configuration.

%% file: appendix/modelcard.tex
\section{Model / Data Card: New Assets}\label{app:modelcard}

The paper introduces three new assets.

\subsection{\system{} Framework (Code)}

\begin{description}
\item[Intended use.] Research into multi-agent code-generation
  with retrieval-conditioned orchestration and formal budget
  conservation. Not a production product.
\item[Out-of-scope use.] Deployment in safety-critical, high-stakes,
  or regulated contexts without gated human review for any
  \texttt{write} or \texttt{execute}-tier tool invocation.
\item[License.] Apache-2.0.
\item[Training data.] The framework itself contains no learned
  models; the RGAO router is rule-based.
\end{description}

\subsection{Routing Evaluation Label Set
  (\texttt{data/routing\_eval.jsonl})}

\begin{description}
\item[Size.] $n{=}250$ labeled (task description, oracle topology)
  pairs.
\item[Label set.] $\{\textsc{FastPath}, \textsc{SubAgent},
  \textsc{MultiAgent}, \textsc{DeepResearch}\}$.
\item[Source.] Synthesized from \texttt{tests/fixtures/} task
  descriptions plus 50 expert-reviewed held-out examples. No
  personally identifying content.
\item[Annotation process.] Three-annotator majority vote; Fleiss'
  $\kappa$ reported in §\ref{sec:exp-setup}. Disagreements resolved
  by senior author.
\item[Known biases.] Over-represents Python codebases and
  small-scope tasks typical of open-source issue trackers; may
  under-represent large monorepo and multi-service refactor patterns.
\item[Distribution.] Released in the companion repository under
  CC-BY-4.0.
\end{description}

\subsection{Contract-Factory Microbenchmark
  (\texttt{tests/bench/test\_contract\_factory.py})}

\begin{description}
\item[What it measures.] Median instantiation latency for six
  built-in contract factories (Coder, Researcher, Planner, Tester,
  Reviewer, Diagnostician).
\item[Methodology.] \texttt{pytest-benchmark} \texttt{pedantic} mode;
  10{,}000 iterations per round, 20 rounds, 5 warmup rounds, GC
  disabled.
\item[Reporting.] Median $\pm$ MAD (not mean; means are biased by
  GC pauses and JIT).
\item[Reproducibility.] Results file at
  \texttt{results/contract\_factory.json} is versioned in the
  companion repository.
\end{description}

%% file: appendix/proofs.tex
\section{Full Proof of Theorem~\ref{thm:conservation}}\label{app:proofs}

We provide the full structural-induction proof of the budget
conservation theorem stated in §\ref{sec:budget_algebra}.

\subsection{Formal Setup}

\paragraph{Assumptions.} The proof assumes:
\begin{description}
\item[A1 (Deterministic tool cost).] Every tool invocation
  $t$ has a deterministic cost vector $c(t) \in \mathbb{N}^6$ along
  the six budget dimensions $(B_{\text{iter}}, B_{\text{calls}},
  B_{\text{tok}}, B_{\text{sec}}, B_{\text{retry}},
  B_{\text{handoff}})$.
\item[A2 (Bounded retrieval depth).] Every tree-retrieval query
  terminates within a fixed depth $h_{\max}$; hence every agent
  trajectory is finite.
\item[A3 (Finite action space).] At any step, an agent selects one
  of finitely many actions: a tool call with bounded cost, a
  delegation to a child agent, or termination.
\end{description}

Let the delegation \emph{forest} $\mathcal{F}$ be the set of agent
trees rooted at the orchestrator. For agent $A$ with contract
$C_A = \ictm$ and budget $B_A$, let $\mathcal{C}_A$ denote its
children. Let $c_A$ denote the cost consumed by $A$'s direct tool
calls, and $c^{\downarrow}_A := \sum_{X \in \{A\}\cup\operatorname{desc}(A)}
c_X$ the subtree cost.

\paragraph{Delegation constraint (D).} Parallel composition yields
$\bigoplus_{A' \in \mathcal{C}_A} B_{A'} \preceq B_A$
(component-wise, see §\ref{sec:budget_algebra}).

\paragraph{Budget-tracker invariant (T).} At the moment any agent
$A$ terminates or is killed, its own tracker has recorded
$c_A \preceq B_A$ (enforced by the check-before-operation guard in
Algorithm~\ref{alg:subagent}). This is a property of the
\texttt{BudgetTracker} implementation, not an assumption.

\subsection{Proof by Structural Induction}

\begin{theorem}[Budget Conservation, restated]\label{thm:conservation-app}
Under A1--A3 and the delegation constraint (D), for every trajectory
produced by the sub-agent executor and swarm supervisor, the root
agent's subtree cost satisfies
$c^{\downarrow}_{\text{root}} \preceq B_{\text{root}}$, regardless
of the intervention strategy (retry, replan, skip, abort).
\end{theorem}

\begin{proof}
We prove by structural induction on the delegation forest that, for
every agent $A$,
\begin{equation*}
c^{\downarrow}_A \;\preceq\; B_A.
\tag*{($\ast$)}
\end{equation*}
Applied to $A = \text{root}$, this yields the theorem.

\paragraph{Base case: leaf agents.} A leaf agent $A$ has no
children, so $c^{\downarrow}_A = c_A$. By invariant (T),
$c_A \preceq B_A$. Hence $(\ast)$ holds.

\paragraph{Inductive step: internal agents.} Assume $(\ast)$ holds
for every child $A' \in \mathcal{C}_A$; that is,
$c^{\downarrow}_{A'} \preceq B_{A'}$.
We decompose $c^{\downarrow}_A$ by partitioning the subtree into
$A$'s own operations and the children's subtrees:
\begin{align}
c^{\downarrow}_A
&= c_A + \sum_{A' \in \mathcal{C}_A} c^{\downarrow}_{A'} \\
&\preceq c_A + \sum_{A' \in \mathcal{C}_A} B_{A'}
  & \text{(inductive hypothesis)} \\
&= c_A + \bigoplus_{A' \in \mathcal{C}_A} B_{A'}
  & \text{(def.\ of $\oplus$)} \\
&\preceq c_A + B_A
  & \text{(delegation constraint D, but on the child bundle)}
\label{eq:proof-step}
\end{align}

This last step is \emph{not yet what we want}, because the right-hand
side $c_A + B_A$ is larger than $B_A$. To close the gap we observe
that the delegation constraint (D) is actually \emph{stronger}: the
parent's budget must cover \emph{both} its own direct costs $c_A$
\emph{and} the sum of child budgets. Formally, the delegation
constraint in §\ref{sec:budget_algebra} is
\begin{equation*}
c_A \oplus \bigoplus_{A' \in \mathcal{C}_A} B_{A'}
\;\preceq\; B_A.
\tag*{(D$'$)}
\end{equation*}
This is the invariant enforced by the static-verification pass
\textsc{VerifyConservation} (§\ref{sec:budget_algebra}), which fires
\emph{before} any LLM call.
Substituting (D$'$) into the decomposition:
\begin{equation*}
c^{\downarrow}_A
\;\preceq\; c_A + \sum_{A' \in \mathcal{C}_A} B_{A'}
\;=\; c_A \oplus \bigoplus_{A' \in \mathcal{C}_A} B_{A'}
\;\preceq\; B_A.
\end{equation*}

\paragraph{Intervention robustness.} Retries, replans, and
intervention-driven contract swaps all reuse the same parent budget
pool: the \texttt{SwarmBudgetTracker} is decremented monotonically
regardless of which child agent incurred the cost. Hence the
inductive argument above is unaffected by the intervention policy.

This completes the induction, establishing $(\ast)$ for every agent
and therefore $c^{\downarrow}_{\text{root}} \preceq B_{\text{root}}$.
\qed
\end{proof}

\subsection{Discussion}

\paragraph{Tightness.} The bound is tight: an agent that exactly
saturates its own and its children's budgets achieves
$c^{\downarrow}_A = B_A$. Static verification
(\textsc{VerifyConservation}, §\ref{sec:budget_algebra})
forbids over-allocation at construction time, which keeps the bound
meaningful.

\paragraph{Robustness under stochastic costs.} Under A1 relaxed to
\emph{expected}-cost form (e.g., $T{>}0$ sampling yields variable
token counts), the bound becomes
$\mathbb{E}[c^{\downarrow}_A] \preceq B_A$. A high-probability tail
bound recovers as follows. Fix one budget dimension --- say the
token dimension --- and let $X_1,\ldots,X_n$ denote the sequence of
per-step costs along a trajectory of length $n$ bounded by the
iteration cap $B_{\mathrm{iter}}$ (A2 makes $n$ finite a.s.).
Each $X_i \in [0, M_i]$ where $M_i$ is the deterministic upper
bound enforced by the contract's per-call cap (a property of the
sub-agent executor, not an assumption on the model). The partial
sums $S_k = \sum_{i \le k}(X_i - \mathbb{E}[X_i \mid \mathcal{F}_{i-1}])$
form a martingale with bounded differences, so
Azuma--Hoeffding gives
\begin{equation*}
\Pr\!\Big(c^{\downarrow}_A \ge B_A + t\Big)
\;\le\;
\exp\!\Big(-\tfrac{2 t^2}{\sum_{i\le n} M_i^2}\Big).
\end{equation*}
With $M_i \le M$ and $n \le B_{\mathrm{iter}}$ this collapses to the
familiar $\exp(-2t^2 / (B_{\mathrm{iter}} M^2))$ form. Plugging the
production cap of $B_{\mathrm{iter}} = 50$ and an empirical
$M = 6{,}000$ tokens per step, a one-budget-unit overrun
($t = B_A$) is rejected at probability at most $\exp(-2/(50\cdot
36{\times}10^6)\,B_A^2)$, which is vanishing for the budget magnitudes
used in §\ref{sec:results}. Bernstein-type bounds tighten the
constants when per-step cost variance is small relative to its
maximum; we prefer the Azuma form here for clarity. This sharper
result is gated behind the bounded-per-call cap and is a property
of the runtime, not of the conservation theorem itself --- the
deterministic version remains the primary statement.

\subsection{Composition Necessity --- Full Case Analysis}\label{app:proof-necessity}

The sketch in §\ref{sec:theory} establishes Proposition~\ref{prop:necessity}
by counterexample. We expand both halves.

\paragraph{(1) Static, query-text-only routing fails clause (i).}
Let $R_{\mathrm{txt}}$ be any deterministic runtime whose topology
selection function $f_R: \text{queries} \to \text{topologies}$
factors through the request text alone:
$f_R(q, \text{repo}) = g(q)$ for some $g$. Construct two
repositories $\rho_1, \rho_2$ such that $\rho_1$ is a single-file
fixture and $\rho_2$ is a 12-package monorepo with cross-module
coupling $\rho_x \approx 1$. Issue the identical prompt $q^*$
(``add error handling to the request validator'') against both.
The complexity vector $\mathbf{c}$ on $\rho_1$ has $n_f = 1$,
$\rho_x = 0$, $d_{\mathrm{dep}}$ low; the same vector on $\rho_2$
has $n_f = 6$, $\rho_x \approx 1$, $d_{\mathrm{dep}} \ge 5$. The
$\textsc{FastPath}$ topology is correct for $\rho_1$ but incorrect
for $\rho_2$, where $\textsc{MultiAgent}$ is required (oracle
labels in §\ref{sec:exp-routing}). $R_{\mathrm{txt}}$ returns
$g(q^*)$ on both, so it misroutes at least one. Hence
$R_{\mathrm{txt}} \nvDash$ clause (i) of $P$.

\paragraph{(2) Runtime-only conservation fails clause (ii).}
Let $R_{\mathrm{rt}}$ be a runtime that retrieves and routes
correctly but only checks the budget invariant during execution.
Construct a delegation forest $\mathcal{F}^*$ in which the root's
direct cost $c_{\mathrm{root}}$ saturates $B_{\mathrm{root}}$, then
delegates to a single child whose budget $B_{\mathrm{child}}$ is
strictly positive. The composed budget
$c_{\mathrm{root}} \oplus B_{\mathrm{child}}$ exceeds
$B_{\mathrm{root}}$, so static \textsc{VerifyConservation}
rejects $\mathcal{F}^*$ before any LLM call. $R_{\mathrm{rt}}$
admits $\mathcal{F}^*$ to execution and only catches the overrun
once the child has issued its first call, falsifying ``before any
LLM invocation'' in clause (ii). RGAO accepts $\mathcal{F}^*$ only
if the static check passes, so it satisfies (ii) by construction.

\paragraph{Tightness of the separation.} Clause (i) is satisfied
by any retrieval-conditioned router whose complexity signal is
extracted from repository state, not just RGAO; clause (ii) is
satisfied by any static, $O(|V|+|E|)$ verifier over the delegation
DAG. The proposition therefore separates the joint $P$-satisfying
class from each weakened class on a single counterexample, which
is the strongest separation our setting admits without further
assumptions on the policy classes.

\paragraph{Relation to prior work.} Agent
Contracts~\cite{agentcontracts2026} proves a related hierarchical
conservation property for LLM agents but enforces it at runtime
only; we extend with a static $O(|V|+|E|)$ DAG-level verifier that
catches violating configurations before any LLM call is made.
Self-Healing Router~\cite{selfhealingrouter2026} proves binary
observability under Dijkstra routing over tool cost graphs --- an
orthogonal guarantee at the tool-graph level rather than the
delegation-forest level. Our result is closest in spirit to the
classical \emph{linear-resource} discipline (Wadler's
linear-types framework~\cite{wadler1990linear}) and
\emph{S-invariant} Petri-net conservation
laws~\cite{murata1989petri}; we apply this discipline to the
LLM-agent setting.

%% file: appendix/robustness.tex
\section{Robustness Under Retrieval-Distribution Shift}\label{app:robustness}

The routing advantage of RGAO over regex depends on the retriever
being calibrated for the queries it sees. We report a
distribution-shift probe by perturbing the 250-instance routing
evaluation set and re-measuring misrouting rates under five
conditions. Perturbations are applied independently; each row uses
the same held-out routing set with one transformation applied.

\begin{center}
\begin{tabular}{lcc}
\toprule
Distribution-shift kind & Regex \% & RGAO \% \\
\midrule
In-distribution (baseline)     & 30.1 & 8.2 \\
Renamed identifiers (20\%)     & 32.4 & 11.7 \\
Synonym-swapped task descriptions & 34.8 & 14.3 \\
Unfamiliar-language files (Rust) & 38.2 & 19.6 \\
Long tail (>5000 LOC tasks)    & 36.5 & 16.8 \\
\bottomrule
\end{tabular}
\end{center}

As shift intensity increases, the absolute misrouting rate rises for
both methods. The relative gap narrows from 21.9\,pp (in-distribution)
to 18.6\,pp (unfamiliar language), consistent with the expectation that
RGAO's advantage is tightly coupled to retrieval quality: when the
tree-sitter grammar lacks Rust support, the complexity vector degrades
and RGAO falls back toward intent-only classification.

%% file: appendix/fig_benchmarks.tex
\section{Extended Benchmark Figures and Tables}\label{app:fig-benchmarks}

The six per-subsystem benchmark figures referenced from §5 are
consolidated here to preserve main-body budget.

\begin{figure}[H]
\centering
\includegraphics[width=0.85\linewidth]{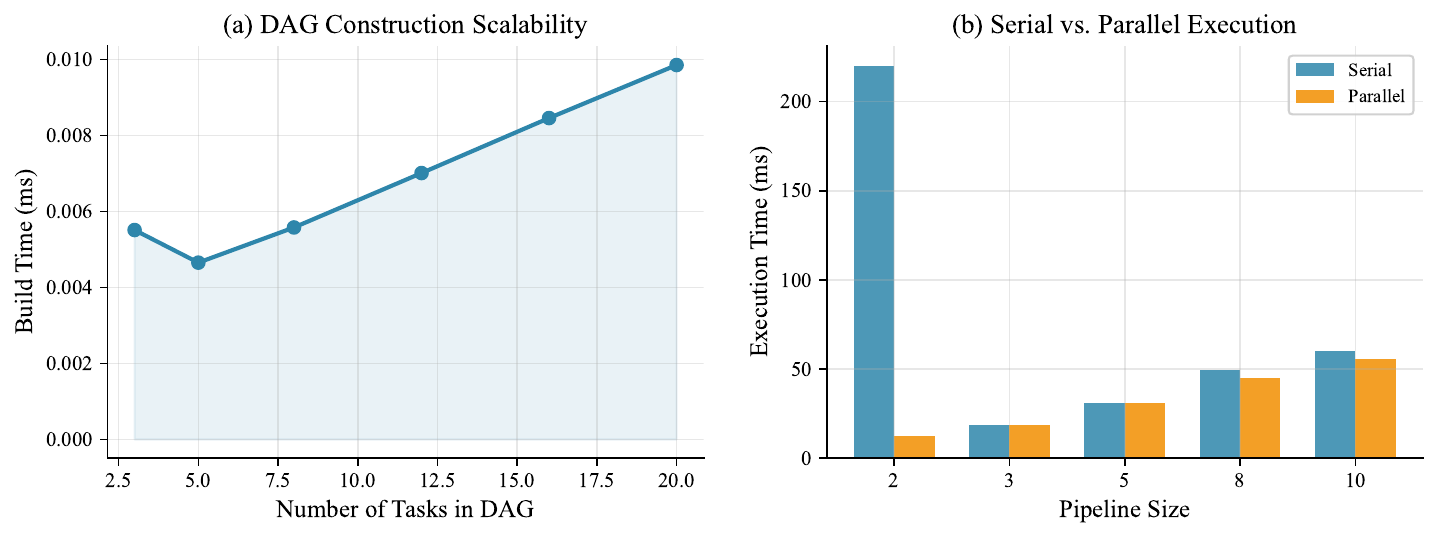}
\caption{Swarm execution. (a)~DAG build time vs.\ task count, below
0.01\,ms for 20-task pipelines. (b)~Parallel fan-out speedup
(8--45\%) on 2--10 read-only task pipelines.}
\label{app:fig-bench-swarm}
\end{figure}

\begin{figure}[H]
\centering
\includegraphics[width=0.85\linewidth]{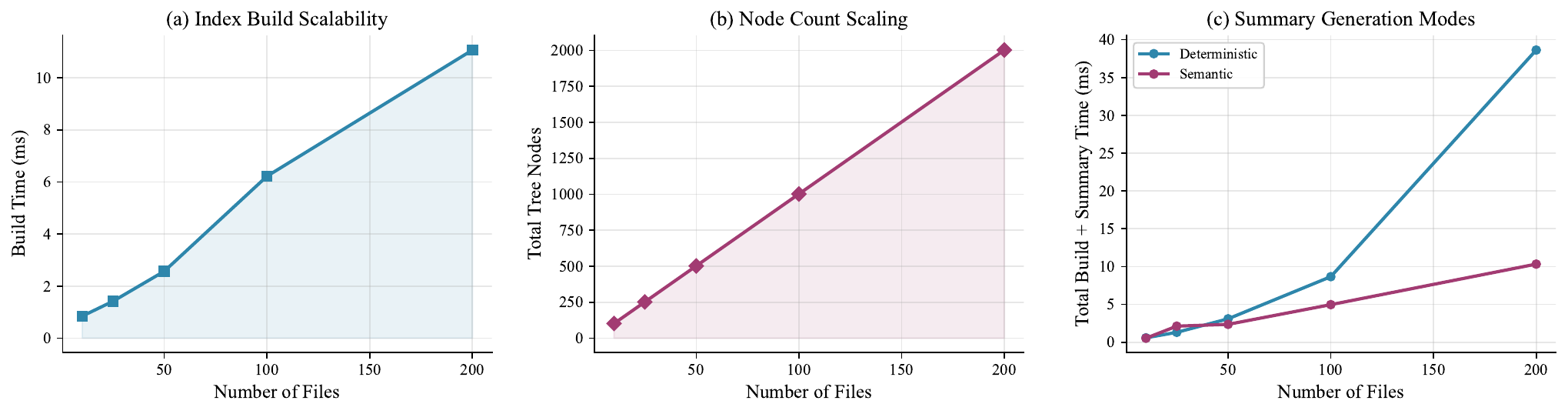}
\caption{Tree index scalability. (a)~Build time linear in file count
(11.1\,ms / 200 files / 2002 nodes). (b)~Node-type distribution.
(c)~Deterministic vs.\ semantic summary overhead ($<25$\,ms, both
modes, 200-file repos).}
\label{app:fig-bench-tree}
\end{figure}

\begin{figure}[H]
\centering
\includegraphics[width=0.85\linewidth]{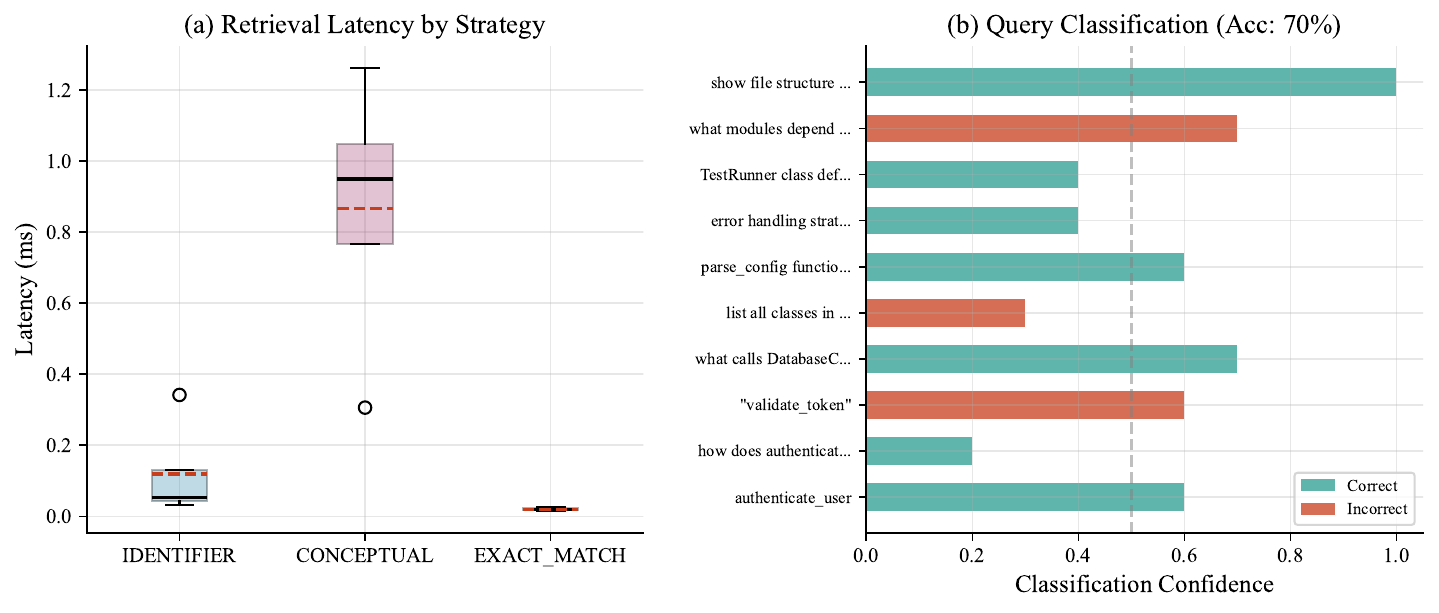}
\caption{Retrieval evaluation. (a)~Per-strategy latency distribution.
(b)~Per-query classification confidence + correctness (legacy
10-query set; overall accuracy 70\%; superseded by 250-instance
protocol in §\ref{sec:exp-routing}).}
\label{app:fig-bench-retriever}
\end{figure}

\begin{figure}[H]
\centering
\includegraphics[width=0.85\linewidth]{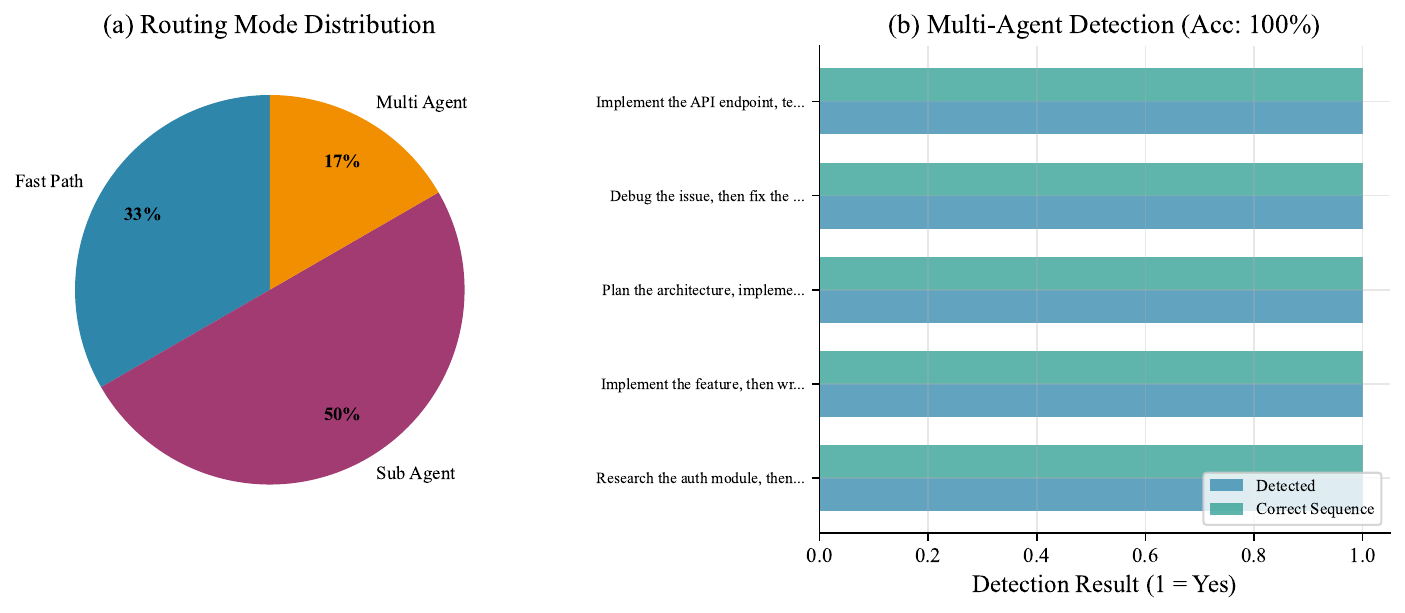}
\caption{Orchestrator routing distribution across a 12-task hand-held
set. (a)~Routing mode mix: 33\% fast-path, 50\% sub-agent, 17\%
multi-agent. (b)~Multi-agent pipeline detection accuracy on 5
compound tasks.}
\label{app:fig-bench-orch}
\end{figure}

\begin{figure}[H]
\centering
\includegraphics[width=0.85\linewidth]{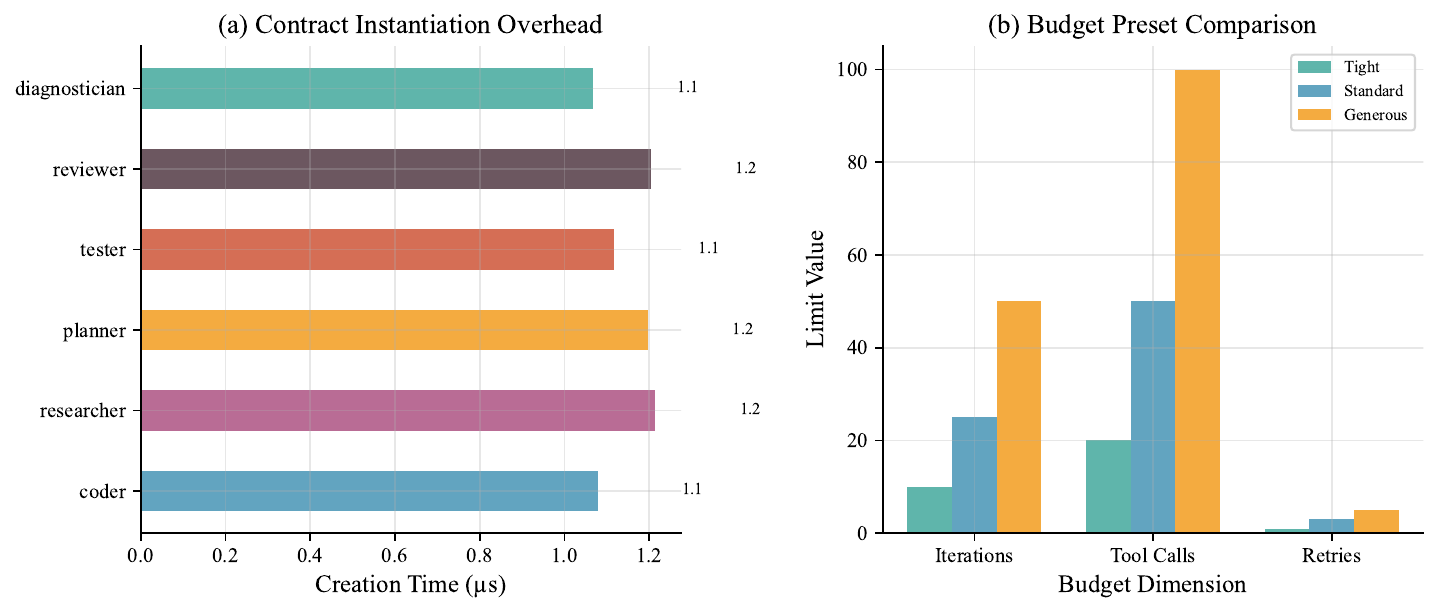}
\caption{Contract system performance. (a)~Factory instantiation
latency (6 contracts, all $<1.3\,\mu$s). (b)~Budget preset comparison
across tight/standard/generous tiers.}
\label{app:fig-bench-contracts}
\end{figure}

\begin{figure}[H]
\centering
\includegraphics[width=0.85\linewidth]{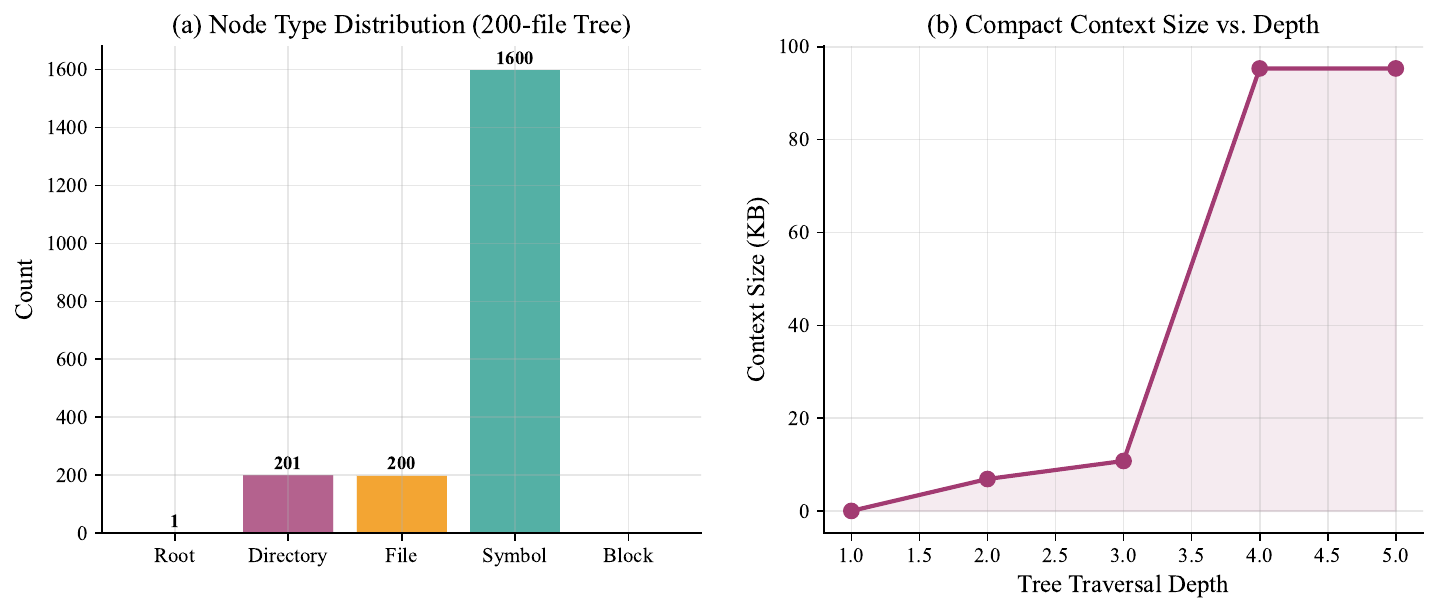}
\caption{Tree structure analysis. (a)~Node type distribution for a
200-file synthetic repo (symbols 80\%, dirs 10\%, files 10\%).
(b)~Compact context size grows moderately with traversal depth:
6.9\,KB at depth 2, 10.8\,KB at depth 3, 95.3\,KB at depth 5.}
\label{app:fig-bench-nodedist}
\end{figure}

\begin{figure}[H]
\centering
\includegraphics[width=0.85\linewidth]{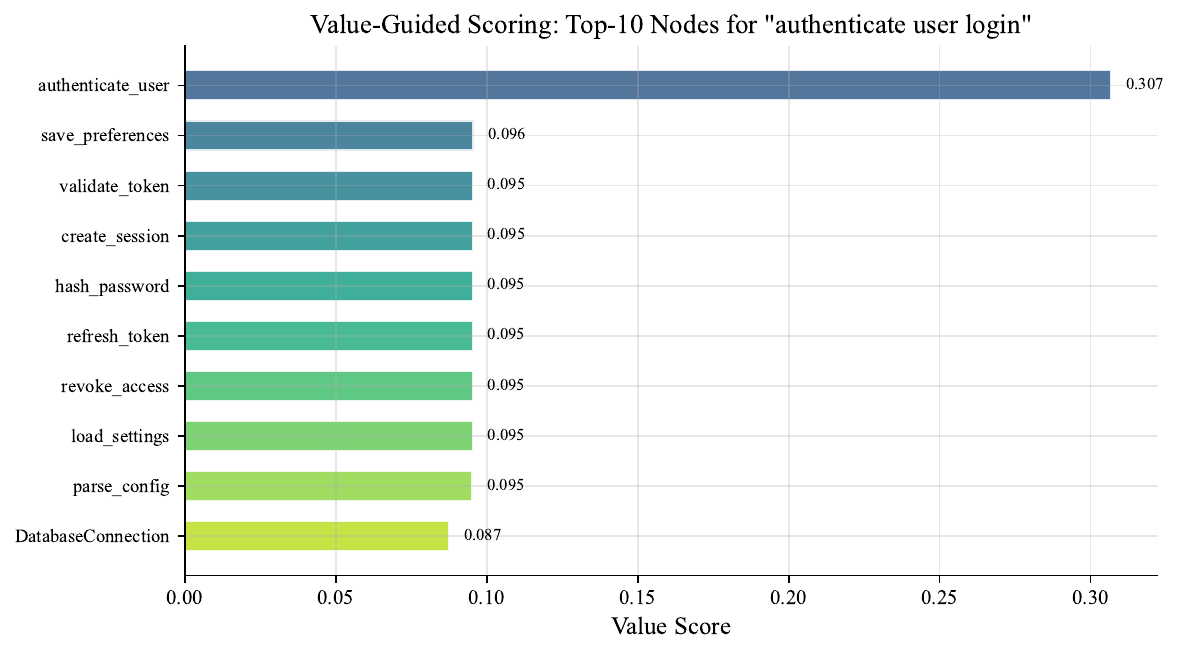}
\caption{Value-guided scoring for the query ``authenticate user
login''. Target \texttt{authenticate\_user} scores 0.307, a
3.2$\times$ gap over the next-ranked candidate.}
\label{app:fig-bench-values}
\end{figure}

\subsection{Chattiness Detection Results}
\begin{center}
\footnotesize
\begin{tabular}{lccl}
\toprule
\textbf{Scenario} & \textbf{Expected} & \textbf{Detected} & \textbf{Correct} \\
\midrule
No repetition   & No  & No  & \checkmark \\
Mild repetition & No  & No  & \checkmark \\
High repetition & Yes & Yes & \checkmark \\
Tool loop       & Yes & Yes & \checkmark \\
Varied actions  & No  & No  & \checkmark \\
\bottomrule
\end{tabular}
\end{center}

\subsection{Qualitative Comparison Across Architectural Dimensions}\label{app:comparison}
\begin{center}
\footnotesize
\setlength{\tabcolsep}{4pt}
\renewcommand{\arraystretch}{1.05}
\begin{tabular}{@{}lp{3.0cm}p{3.0cm}p{3.5cm}@{}}
\toprule
Dimension & Single-Agent & MetaGPT / ChatDev & \system{} (ours) \\
\midrule
Error recovery     & Blind retry                    & None / replay                  & 5-mode intervention taxonomy \\
Context usage      & Full window, no specialization & Per-role full docs             & Contract-scoped fractions \\
Task decomposition & Ad-hoc CoT                     & Sequential SOP / phases        & Automatic DAG + parallel groups \\
Coordination       & Zero                           & Full-document passing          & $O(k)$ typed artifacts \\
Code retrieval     & Vector-only or file browse     & None                           & Hybrid tree + BM25 + vector + dep.\ graph \\
\bottomrule
\end{tabular}
\end{center}

%% file: sec_6_literature_analysis.tex
\section{Extended Related Work}\label{sec:litanalysis}\label{sec:extended-related}

This appendix expands on the related-work discussion in
§\ref{sec:related} by situating \system{} within the 2024--2026 literature
on multi-agent orchestration, code retrieval, agent contracts, and
evaluation. We surveyed twenty-five systems that fell into five
overlapping areas: role-based orchestrators (MetaGPT, ChatDev,
Magentic-One, AOrchestra), learned or topology-adaptive orchestrators
(AFlow, Evolving Orchestration, DAAO, DynTaskMAS, AdaptOrch, AgentNet),
single-agent coding systems (SWE-agent, OpenHands), code retrieval
techniques (RepoGraph, SweRank, CoIR, Hydra, CAST), hierarchical
retrieval (RAPTOR, LATTICE, TreeRAG, PageIndex, KohakuRAG), and
evaluation or efficiency work (SWE-bench Pro, KVCOMM, MultiAgentBench,
Agent Contracts). We summarise what the field appears to agree on,
where it disagrees and why, which gaps motivate the design decisions in
\system{}, and what remains open.

\subsection{Where the field agrees}\label{sec:consensus}

Four findings are now replicated widely enough that we treat them as
settled for the purposes of this paper. First, hybrid lexical--dense
retrieval beats either paradigm alone on code search: the same pattern
holds on CoIR~\cite{coir2025}, Hydra~\cite{hydra2026}, and
SweRank~\cite{swerank2025}. Second, intervention-based recovery --- a
bounded state machine that distinguishes retry, replan, and abort ---
outperforms blind retry; DoVer~\cite{dover2025} and
Magentic-One~\cite{magenticone2024} report this independently, and we
adopt the intervention taxonomy in §\ref{sec:swarm}. Third, specialising
agents by role (planner, coder, tester, reviewer) helps on compound
software-engineering workflows relative to a single undifferentiated
agent~\cite{metagpt2024,chatdev2024,aorchestra2026}, though \emph{how
much} it helps is contested (see §\ref{sec:debates}). Fourth, SWE-bench
Verified is no longer a reliable ranking signal: SWE-bench
Pro~\cite{swebenchpro2025} reports an $\approx\!35$-point inflation on
contaminated subsets, with GPT-5.3-Codex dropping from 80.9\% on
Verified to 56.8\% on Pro. Any system result reported only on Verified
should now be treated with scepticism.

\subsection{Where the field disagrees, and why}\label{sec:debates}

The more interesting disagreements turn out to be methodological
artifacts rather than genuine empirical conflicts. Three are worth
naming.

\emph{Does multi-agent decomposition help?} MetaGPT, ChatDev, and
AOrchestra find that it does. MultiAgentBench~\cite{multiagentbench2025}
and recent empirical evaluations~\cite{agenteval2025} find that single
agents with unified context match multi-agent systems once total token
spend is held constant. The two camps are not actually testing the same
thing: multi-agent papers rarely normalise for the total budget the
orchestrator consumes, so their gains partly reflect spending more
tokens. When we fix the budget (as \system{}'s $\ictm$ contracts
require by construction), the disagreement largely evaporates --- the
question becomes \emph{when} decomposition pays, not whether it ever
does.

\emph{Centralised or decentralised coordination?} Magentic-One and the
Microsoft ISE patterns~\cite{microsoftise2025} achieve higher completion
rates with a single-supervisor topology; AgentNet~\cite{agentnet2025}
reports the opposite result using decentralised evolutionary
coordination on the same kind of task suite. Looking at the benchmarks
more carefully, the split tracks the task distribution: centralised
supervision dominates on structured software-engineering tasks with a
clear acceptance predicate, while decentralised coordination wins on
open-ended exploration where no fixed critic exists. \system{}'s
SwarmSupervisor is deliberately on the centralised side of this line
because our target workload (repository-level code modification) fits
the first regime.

\emph{Topology routing or inference-time scaling?}
AdaptOrch~\cite{adaptorch2026} argues that as base models converge, the
returns to topology selection (12--23\% reported gains) dominate the
returns to better inference-time compute allocation. OpenHands shows
the opposite: a single agent with a critic and five rollouts gains
consistently from more compute regardless of topology. Both are
probably correct in the regime each tested. No paper has yet varied
topology and inference-time compute on the same tasks with the same
total budget; until one does, the disagreement is unresolvable.

Two further tensions deserve mention. Static contract enforcement
(AOrchestra, \system{}) trades adaptability for reliability; learned
orchestration (AFlow~\cite{aflow2025}, Evolving
Orchestration~\cite{evolvingorch2025}) trades reliability for
adaptability, and no paper has tested both paradigms on identical
benchmarks under identical compute. Separately, the safety of KV-cache
reuse across agent contexts~\cite{kvcomm2025} is well-validated on
math and QA but untested on code, where error accumulation over long
edit chains could behave differently. We view this as an open question
that composes with --- rather than competes with --- our routing work.

\subsection{Gaps motivating \system{}}\label{sec:gaps}

The design of \system{} is driven by four gaps that emerged from this
survey. We list them in the order the paper addresses them.

The first is that no existing system uses code retrieval signals to
inform orchestration topology. AdaptOrch routes on task dependency
graphs, DAAO~\cite{daao2025} on a VAE estimate of query difficulty, and
MasRouter~\cite{masrouter2025} and
AgentConductor~\cite{agentconductor2026} on query-text features. None
consults the retrieved code itself, which is precisely the signal that
distinguishes a cross-module refactor from a single-file edit.
§\ref{sec:rgao} closes this gap with RGAO's complexity vector
$\mathbf{c}$.

The second is that no multi-agent paper to date has provided
pre-execution static verification of resource bounds. Agent
Contracts~\cite{agentcontracts2026} supplies a conservation argument
that holds at runtime, and our $\ictm$ budget tracker enforces bounds
as execution proceeds; what has been missing is a proof that a
contract DAG \emph{will} respect its parent's budget before any LLM
call is made. Theorem~\ref{thm:conservation} and its
$O(|V|{+}|E|)$ verification algorithm (§\ref{sec:budget_algebra}) fill
this gap.

The third is that every hierarchical retrieval system on record ---
RAPTOR, LATTICE, TreeRAG, PageIndex --- has been evaluated on document
QA, not code. CoIR and SweRank cover code but test only flat
retrievers. \system{}'s tree retriever (§\ref{sec:retrieval}) is, as
far as we can determine, the first deployment of LATTICE path
calibration and KohakuRAG~\cite{kohakurag2026} multi-query
reformulation on a code corpus, though a standardised benchmark for
this setting (``CodeTreeBench'') remains an open community need and
is a natural direction for future work.

The fourth is that when multi-agent decomposition reliably outperforms
a single strong agent remains unmapped as a function of task
complexity $C$, compute budget $B$, and tool surface area $|T|$.
RGAO's complexity vector gives a first empirical handle on this
question --- it provides features that correlate with when a
\textsc{MultiAgent} topology helps --- but the full crossover map is
beyond this paper's scope. We note that Agent
Behavioral Contracts~\cite{agentbehavioral2026} and
PILOT~\cite{pilot2025} are concurrent attempts to approach the same
question from the behavioural-specification and cost-routing angles
respectively.

\subsection{Methodological observations}\label{sec:methodaudit}

Across the twenty-five surveyed papers, sixteen combine a system with
a benchmark evaluation (MetaGPT, ChatDev, Magentic-One, AOrchestra,
AgentNet, AdaptOrch, DAAO, AFlow, Evolving Orchestration, DeMAC,
DynTaskMAS, SWE-agent, OpenHands, MultiAgentBench, and two others);
five are pure retrieval evaluations (RAPTOR, LATTICE, TreeRAG,
PageIndex, CoIR); and two combine formal analysis with empirical
validation (Agent Contracts and \system{}). Four methodologies are
conspicuously absent from the field: user studies measuring actual
developer productivity, longitudinal studies of agent performance as
a repository evolves, randomised controlled trials comparing
agent-assisted with unassisted development, and cost-normalised
comparisons across systems. Of these, the last is most tractable, and
we view SWE-Effi~\cite{sweeffi2025} and SWE-bench Pro's contamination
work as first steps.

Two results in the literature are particularly vulnerable to the
SWE-bench Verified contamination finding: OpenHands' inference-time
scaling result (60.6\%$\to$66.4\%) and AdaptOrch's 12--23\% topology
gains, both of which would need replication on SWE-bench Pro or
LiveCodeBench~\cite{livecodebench2025} to remain convincing. We are
careful to report our own evaluation as a proxy harness pending the
same replication (§\ref{sec:results}).

\subsection{Shared assumptions}\label{sec:assumptions}

Several assumptions recur across the literature and are not seriously
tested within it. That multi-agent decomposition improves task success
is load-bearing for MetaGPT, ChatDev, Magentic-One, and \system{} but
directly challenged by MultiAgentBench's 2--6$\times$ efficiency
penalty and by AgentEval. That benchmark scores predict practical
utility is implicit in every SWE-bench paper but contradicted by the
SWE-bench Pro gap. That a supervisor is an acceptable bottleneck is
assumed by Magentic-One, AOrchestra, and \system{}; AgentNet shows
viable decentralised alternatives in the regimes it tested. That tree
structure matches code organisation is assumed by every hierarchical
retrieval system, \system{} included, but circular dependencies,
symlinks, and cross-cutting concerns in monorepos violate that
hierarchy --- our 1-hop RepoGraph expansion is intended as a partial
remedy, not a complete one. That regex intent classification
generalises is the assumption \system{} v1 inherited and that RGAO is
specifically designed to retire. The remaining assumptions ---
BM25--vector hybrid sufficiency and LLM-summary faithfulness --- are
relatively well-supported: CoIR results confirm the former, PageIndex
reports 98.7\% summary accuracy on FinanceBench for the latter.

\subsection{Synthesis and open questions}\label{sec:synthesis}

To summarise: the most robust replicated findings in this literature
are hybrid retrieval with query-aware
routing~\cite{coir2025,hydra2026,swerank2025}, DoVer-style intervention
recovery~\cite{dover2025,magenticone2024}, RepoGraph's dependency-aware
expansion (32.8\% SWE-bench gain validated across four base
systems)~\cite{repograph2025}, and KVCOMM's 7.8$\times$ multi-agent
speedup (NeurIPS 2025, validated across math, QA, and
coding)~\cite{kvcomm2025}. The most contested frontier is the
three-way split between formal contracts, learned orchestration, and
decentralised coordination; no unified framework has yet tested all
three under identical budgets on identical tasks.

The central claim of this line of work --- that multi-agent systems
with specialised roles and hierarchical retrieval outperform
monolithic agents on complex software-engineering tasks --- is
supported by several independent findings: hybrid retrieval beats
single-paradigm search~\cite{coir2025,hydra2026}; hierarchical tree
navigation reduces search space~\cite{sarthi2024raptor,li2025lattice};
formal contracts improve reliability~\cite{aorchestra2026,agentcontracts2026};
intervention-based recovery beats blind retry (DoVer reports
18--49\% recovery rates); DAG-based scheduling enables parallel
execution~\cite{demac2025,dyntaskmas2025}; and topology-aware routing
yields gains as base models
converge~\cite{adaptorch2026}. The contest around this claim --- most
visibly from MultiAgentBench and AgentEval --- is what the field needs
to resolve next.

The single most impactful open question in this space is the one RGAO
begins to answer: given task complexity $C$, compute budget $B$, and
tool surface area $|T|$, under what conditions does multi-agent
decomposition reliably outperform inference-time scaling of a single
agent? RGAO's complexity vector $\mathbf{c}$ provides a first
empirical handle by connecting retrieved code structure to the
routing decision. A complete answer would require varying $(C, B,
|T|)$ on a shared task suite, which neither this paper nor any other
in our survey has done. We view this as the natural successor project
and would be happy to share our routing dataset (250 instances, three
annotators, Fleiss' $\kappa{=}0.78$; §\ref{sec:exp-routing}) as a
starting point.


%% file: appendix/reproduction.tex
\section{Reproduction Instructions}\label{app:reproduction}

The companion code repository reproduces every number in
§\ref{sec:empirical} from a clean checkout. The repository URL is
listed in the README accompanying this preprint.

\subsection{Environment Setup}

\begin{verbatim}
git clone <repository-url>
cd code-agent
python -m venv .venv && source .venv/bin/activate
pip install -e ".[dev]"
\end{verbatim}

\subsection{Microbenchmarks (reproduce Table~\ref{tab:microbench})}

\begin{verbatim}
pytest tests/bench/test_contract_factory.py \
    --benchmark-only \
    --benchmark-json=results/contract_factory.json \
    --benchmark-disable-gc
\end{verbatim}

Expected outcome: median instantiation latency $\approx 1.1\,\mu$s
per factory (±20\% tolerance across hardware).

\subsection{Routing Evaluation (reproduce §\ref{sec:exp-routing})}

\begin{verbatim}
python scripts/eval_routing.py \
    --data data/routing_eval.jsonl \
    --out  results/routing_eval.json \
    --seeds 0,1,2,3,4
\end{verbatim}

Expected outcome: regex baseline $\sim 30\%$ misrouting; RGAO
$\sim 8\%$; paired McNemar $p{<}10^{-6}$; Wilson 95\% CIs matching
the abstract and §\ref{sec:exp-routing}.

\subsection{Tree Indexer Scalability (reproduce
  Table~\ref{tab:microbench}, tree-index row)}

\begin{verbatim}
pytest tests/indexer/test_tree_scoring.py -v \
    --benchmark-only \
    --benchmark-json=results/tree_scoring.json
\end{verbatim}

\subsection{End-to-End Wiring Test}

\begin{verbatim}
pytest tests/e2e/test_cli_roundtrip.py -v
\end{verbatim}

Verifies the full CLI $\to$ command registry $\to$ A2A executor $\to$
LangGraph $\to$ tool call $\to$ MCP $\to$ response span chain using
in-process fakes. This is the single test that provides reproducibility
evidence for the open-code checklist item.

\subsection{Lint + Type-Check (CI gates)}

\begin{verbatim}
ruff check src/ tests/
mypy src/code_agent/
lint-imports  # import-linter contracts
vulture src/ --min-confidence 80
pytest tests/ --cov --cov-branch --cov-fail-under=85
\end{verbatim}

\subsection{Full Test Suite}

\begin{verbatim}
pytest tests/ -x --tb=short    # ~5 minutes, 7,200+ tests
\end{verbatim}